%% file: main.tex
\documentclass{article}

\usepackage{microtype}
\usepackage{graphicx}
\usepackage{subfigure}
\usepackage{booktabs} 
\usepackage{adjustbox}
\newcounter{1}
\newcounter{2}
\newcounter{3}
\newcounter{4}
\usepackage[noend]{algpseudocode}
\usepackage{algorithm}
\usepackage{algorithmicx}

\usepackage{hyperref}
\newcommand{\xhdr}[1]{{\noindent\bfseries #1}.}

\usepackage[accepted]{icml2023}

\usepackage{amsmath}
\usepackage{amssymb}
\usepackage{mathtools}
\usepackage{amsthm}
\usepackage{pifont}

\DeclareMathOperator*{\argmax}{arg\,max}

\usepackage[capitalize,noabbrev]{cleveref}

\usepackage{thm-restate}
\theoremstyle{plain}
\newtheorem{theorem}{Theorem}[section]

\newtheorem{corollary}[theorem]{Corollary}
\theoremstyle{definition}

\theoremstyle{remark}

\usepackage[textsize=tiny]{todonotes}

\icmltitlerunning{Answering Complex Logical Queries on Knowledge Graphs via Query Computation Tree Optimization}

\begin{document}

\twocolumn[
\icmltitle{Answering Complex Logical Queries on Knowledge Graphs \\ via Query Computation Tree Optimization}

\icmlsetsymbol{equal}{*}

\begin{icmlauthorlist}
\icmlauthor{Yushi Bai}{yyy}
\icmlauthor{Xin Lv}{yyy}
\icmlauthor{Juanzi Li}{yyy}
\icmlauthor{Lei Hou}{yyy}
\end{icmlauthorlist}

\icmlaffiliation{yyy}{Department of Computer Science and Technology, BNRist; KIRC, Institute for Artificial Intelligence; Tsinghua University, Beijing 100084, China}

\icmlcorrespondingauthor{Lei Hou}{houlei@tsinghua.edu.cn}

\icmlkeywords{Machine Learning, ICML}

\vskip 0.3in
]

\printAffiliationsAndNotice{}

\input{000abstract.tex}

\input{010intro}

\input{020related}

\input{030model}

\input{040experiments}

\input{050conclusion}

\section*{Limitation}
Although our method provides more accurate answers for logical queries on KGs, it still suffers from two limitations.
The first pertains to scalability: the process of obtaining a pre-computed adjacency matrix can be both time-consuming and resource-intensive, especially for larger KGs. 
We provide some insights and future directions to mitigate such issue in Appendix~\ref{app:discussion}.
Another challenge lies in the restricted nature of supported query types.
To elaborate, our optimization method is only compatible with tree-like query structures, i.e., the query computation tree.
Unfortunately, it lacks the ability to accommodate cyclic query structures, or queries that contain more than one answer variables.

\section*{Acknowledgement}
This work is supported by a grant from the Institute for Guo Qiang, Tsinghua University (2019GQB0003), and the NSFC Youth Project (62006136).
We gracefully thank all our anonymous reviewers for their fruitful suggestions.

\bibliography{anthology}
\bibliographystyle{icml2023}

\newpage
\appendix

\input{060appendix}

\end{document}

%% file: 000abstract.tex
\begin{abstract}
Answering complex logical queries on incomplete knowledge graphs is a challenging task, and has been widely studied.
Embedding-based methods require training on complex queries and may not generalize well to out-of-distribution query structures.
Recent work frames this task as an end-to-end optimization problem, and it only requires a pretrained link predictor.
However, due to the exponentially large combinatorial search space, the optimal solution can only be approximated, limiting the final accuracy.
In this work, we propose QTO (\textbf{Q}uery Computation \textbf{T}ree \textbf{O}ptimization) that can efficiently find the exact optimal solution.
QTO finds the optimal solution by a forward-backward propagation on the tree-like computation graph, i.e., query computation tree.
In particular, QTO utilizes the independence encoded in the query computation tree to reduce the search space, where only local computations are involved during the optimization procedure.
Experiments on 3 datasets show that QTO obtains state-of-the-art performance on complex query answering, outperforming previous best results by an average of 22\%.
Moreover, QTO can interpret the intermediate solutions for each of the one-hop atoms in the query with over 90\% accuracy.
The code of our paper is at \url{https://github.com/bys0318/QTO}.

\end{abstract}

%% file: 010intro.tex
\section{Introduction}
\label{sec:intro}
Knowledge graph (KG) stores structural knowledge in the form of triplet, which connects a pair of entities (nodes) with a relational edge.
Knowledge graph also supports a variety of downstream tasks, in this paper, we focus on complex logical query answering on knowledge graphs~\cite{hamilton2018embedding, ren2020beta}, which is a fundamental and practical task.

Complex logical queries can be represented with First-Order Logic (FOL) that includes conjunction ($\land$), disjunction ($\lor$), negation ($\lnot$), and existential quantifier ($\exists$).
A more straightforward way is to represent their computation graphs as Directed Acyclic Graphs (DAGs), and solve them by traversing the KG and assigning viable entities to the intermediate variables according to their structures~\cite{dalvi2007efficient,zou2011gstore}.
For example, we show the FOL and DAG representations of a complex logical query ``\emph{Where was the physicist who won the Nobel prize in 1921 born?}'' in Figure~\ref{fig:example}.

\begin{figure}[t]
\centering
\includegraphics[width=1\linewidth,trim=45 70 20 50,clip]{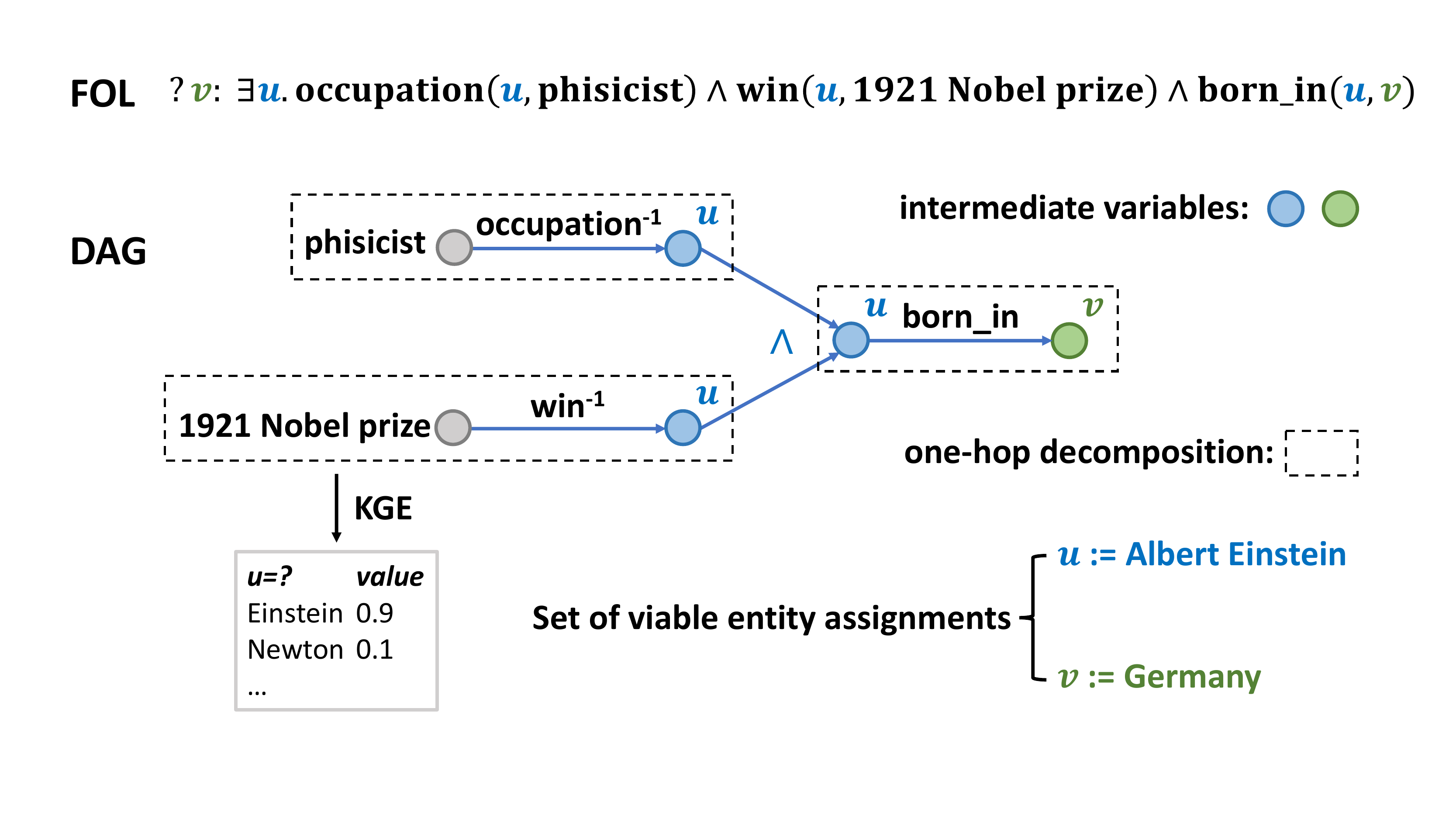}
\caption{FOL and DAG expressions of a complex logical query. The complex query can be decomposed to three one-hop queries, as marked by the dashed boxes. KGE model can answer the one-hop queries with likelihood scores.}
\label{fig:example}
\end{figure}

Real-world KGs often suffer from incompleteness, making it impossible to answer a complex query with missing links by traversing the KG.
Inspired by the success of knowledge graph embedding (KGE)~\cite{bordes2013translating,dettmers2018convolutional,bai2021cone} on answering one-hop KG queries, a line of research proposes to answer complex logical queries by learning embeddings for the intermediate variables~\cite{hamilton2018embedding,ren2019query2box,ren2020beta,zhang2021cone,chen2022fuzzy,zhu2022neural}.
Typically, these methods need to be trained on millions of generated complex logical queries, with large training time overhead, and they struggle to generalize to out-of-distribution (OOD) query structures.
Moreover, they lack necessary interpretability for what entities the intermediate variables stand for.

To these ends, CQD~\cite{arakelyan2021complex} proposes an end-to-end optimization framework, which aims to find a set of entity assignments that maximizes the truth score of a FOL query based on the truth value of each one-hop atom provided by a KGE link predictor (as illustrated in Figure~\ref{fig:example}).
Hence, CQD can explicitly interpret intermediate variables, and it does not require training on complex queries.
However, the search space is exponential to the number of intermediate variables, and CQD alternatively uses restricted discrete search or continuous approximation to \textbf{approximate} the optimal solution. CQD's accuracy is limited due to the approximation.

In this paper, we propose QTO (\textbf{Q}uery Computation \textbf{T}ree \textbf{O}ptimization), which can efficiently find the \textbf{theoretically optimal} solution.
We show that the search space can be greatly reduced by utilizing the independence encoded in the tree-like computation DAG of the query, i.e., query computation tree.
Specifically, QTO recursively maximizes the likelihood of the assignments on the subqueries rooted at variable nodes, and obtains the optimal solution via a forward propagation on the query computation tree.
This divide-and-conquer strategy allows it to efficiently find the exact optimal solution, since only local computations on the tree are involved.
Moreover, a backward pass with $\arg\max$ operation returns the most likely entity assignments for the complex query, bringing interpretability to QTO.
We also provide a theoretical proof for the optimality of our method.

We demonstrate the effectiveness of QTO on complex logical query answering on 3 standard datasets.
Results show that our method outperforms CQD by 30.8\% on average.
More amazingly, even without training on complex queries, our method achieves an average gain of 13.5\% on existential positive first-order queries and 37.5\% on queries with negation, compared to previous state-of-the-art method.
Nevertheless, QTO is the first neural method that can always output the correct answer for easy queries (no missing links along the path), as guaranteed by our theoretical optimality.
For interpretability, QTO can find a set of viable entity assignments with over 90\% success rate.

%% file: 020related.tex
\section{Related Work}
\label{sec:related}

\xhdr{Knowledge Graph Completion}
Knowledge graph completion aims to infer missing relational links (one-hop queries) in an incomplete KG.
A popular approach for such a challenge is knowledge graph embedding (KGE)~\cite{bordes2013translating,trouillon2016complex,dettmers2018convolutional,sun2019rotate,balavzevic2019tucker,chami2020low,bai2021cone}, which learns to embed entities and relations into vectors, and measures the likelihood of a triplet by a defined scoring function over the corresponding vectors.
The embeddings are learned by optimizing the scoring function such that true triplets obtain higher likelihood scores than false triplets.
Other methods for KG completion includes multi-hop reasoning~\cite{das2018go,lin2018multi,lv2019adapting,bai2022squire}, rule-learning~\cite{yang2017differentiable,chris2019anyburl} and GNNs~\cite{schlichtkrull2018modeling,zhu2021neural}.
In this work, we adopt a pretrained KGE model to calculate the truth values of one-hop queries.

\xhdr{Complex Logical Query Answering}
Complex logical queries are one-hop KG queries combined by logical operators.
Compared to KG completion, it also requires the ability to model sets of entities and logical relationships between sets~\cite{guu2015traversing,hamilton2018embedding}.
Existential Positive First-Order (EPFO) queries include existential qualifier ($\exists$), conjunction ($\land$), and disjunction ($\lor$), while the more general First-Order Logic (FOL) queries also include negation ($\lnot$).
Previous embedding-based methods represent sets of entities as geometric shapes~\cite{hamilton2018embedding,ren2019query2box,zhang2021cone,choudhary2021self}, or probabilistic distribution~\cite{ren2020beta}, and predict the answer by locating nearest neighbors to the answer set representation.
However, embedding-based methods usually lack interpretability as there is no explicit mapping between the embedding and the set of entities.
Moreover, the set representation quality may be compromised when the set is large.
To this end, some works combine more interpretable fuzzy logic to tackle complex query answering~\cite{chen2022fuzzy,zhu2022neural}.
GNN-QE~\cite{zhu2022neural} decomposes the query into relation projections and logical operations over fuzzy sets, and learns a GNN to perform relation projections.

However, training on complex queries is required for all aforementioned methods, which limits their generalization ability to more complicated query structures, and prevents these methods from being improved by utilizing more powerful KG completion models.
CQD~\cite{arakelyan2021complex} proposes an optimization-based framework for answering complex queries without training.
It utilizes a pretrained KGE model to score each query atom, and aggregate all atom scores to evaluate the likelihood of a set of variable assignments.
Since there are exponentially many variable assignments, CQD proposes two strategies to approximate the optimal solution: CQD-beam which uses beam search to generate a sequence of entity assignments, and CQD-CO which optimizes directly on the continuous embeddings.
CQD losses its accuracy to the approximation: for example, if the number of viable entity assignments on an intermediate variable surpasses the beam size, then many viable entity assignments will be cut-off due to the beam restriction.
By contrast, our method can guarantee a theoretically optimal solution under the same problem formulation, which is superior to CQD, and our results also show superiority in both accuracy and efficiency.

%% file: 030model.tex
\section{Methodology}
\label{sec:model}

\subsection{Preliminaries}
\xhdr{Knowledge Graphs and Knowledge Graph Embeddings}
Knowledge graph $\mathcal{G}=(\mathcal{V}, \mathcal{E})$ contains a set of entities (vertices) $\mathcal{V}$ and a set of relations $\mathcal{R}$. Each directed edge $(h, r, t)\in\mathcal{E}$ represents a triplet fact, where head and tail entities $h, t\in \mathcal{V}$ and relation $r\in \mathcal{R}$.
KGE models learn mappings $\mathcal{V}\rightarrow\mathbb{R}^{d_\mathcal{V}}, \mathcal{R}\rightarrow\mathbb{R}^{d_\mathcal{R}}$ from entities and relations to vectors, and score the likelihood of a triplet $(h, r, t)$ via a functional $f_r(h, t): \mathbb{R}^{d_\mathcal{V}}\times\mathbb{R}^{d_\mathcal{R}}\times\mathbb{R}^{d_\mathcal{V}}\rightarrow\mathbb{R}$ parameterized by $\mathbf{h}, \mathbf{r}, \mathbf{t}$.
For example, in ComplEx~\cite{trouillon2016complex}, $f_r(h, t)=\text{Re}(\langle \mathbf{r}, \mathbf{h}, \bar{\mathbf{t}}\rangle)$ where $\langle\cdot\rangle$ denotes the generalized dot product, $\bar{\cdot}$ denotes conjugate for a complex vector, and $\text{Re}(\cdot)$ denotes taking the real part of a complex vector.

\xhdr{FOL Queries}
Following the definition in~\cite{ren2020beta}, a FOL query $q$ can be written in the following disjunctive normal form:
\begin{align*}
q[v_?]=v_?.\ \exists v_1, \dots, v_N\in \mathcal{V}:\, &(e^1_{1}\land\dots\land e^1_{m_1})\lor\dots\lor \\&(e^n_{1}\land\dots\land e^n_{m_n})
\end{align*}
where $v_1,\dots,v_N$ are variables and $v_?\in\{v_1, \dots, v_N\}$ is the answer variable. Each literal $e^i_{j}$ denotes the truth value of an atomic relational formula on relation $r$ between two entities, or its negation, respectively:
\begin{align*}
e^i_{j} = 
\begin{cases}
    r(c, v)\text{ or }r(v', v) \\
    \lnot r(c, v)\text{ or }\lnot r(v', v)
\end{cases}
\end{align*}
where $v, v'$ are variables, and $c$ is a constant (anchor) entity.

The goal of query answering is to find a viable variable assignment that renders $q$ true. 
KG incompleteness brings uncertainty into the expression, thus $e^i_j$ is no longer a binary variable.
Instead, it represents the likelihood that the relationship holds true, and this generalized truth value is in $[0, 1]$.
To this end, following \cite{arakelyan2021complex}, we formalize this as an optimization problem:
\begin{align*}
    q[v_?]=v_?.\ v_1, \dots, v_N 
    = &\argmax_{v_1, \dots, v_N\in \mathcal{V}}\, (e^1_{1}\top\dots\top e^1_{m_1})\bot \\
    &\dots\bot (e^n_{1}\top\dots\top e^n_{m_n})
\end{align*}
where $e^i_j\in[0, 1]$ is scored by a pretrained KGE model according to the likelihood of the atomic formula.
$\top$ and $\bot$ are generalizations of conjunction and disjunction over fuzzy logic on $[0, 1]$, namely \emph{t-norm} and \emph{t-conorm}~\cite{klir1995fuzzy,klement2013triangular,hajek2013metamathematics}.
In this paper, we use \emph{product t-norm}, where given $a, b\in [0, 1]$: $\top(a, b)=a\cdot b,\ \bot(a, b)=1-(1-a)(1-b)$.

\xhdr{Query Computation Tree}
We can also derive the computation graph of a FOL query (as shown in Figure~\ref{fig:qto}), where variables are represented as nodes, connected by the following four types of directional edges.
Each atomic formula can be represented with \emph{relational projection} ($r(v',v)$) and \emph{anti-relational projection} ($\lnot r(v',v)$), while logical operators serves as \emph{intersection} ($\land$) and \emph{union} ($\lor$) that merge the branches.
In this paper, we consider a subclass of FOL queries whose computation graphs are trees, namely \emph{query computation trees}.
The answer variable and the constant entities in the query correspond to root and leaf nodes in the query computation tree.
Each edge in the query computation tree points from the child node to the parent node.
It can be recursively deduced that the subtree rooted at any non-leaf node in the tree corresponds to a subquery. 
In Appendix~\ref{app:conversion}, we provide a systematical procedure for transforming a FOL query to its query computation tree.

\xhdr{Optimization on Query Computation Tree}
We apply the optimization formalization on a FOL query to its query computation tree.
Let $T(v)\in[0, 1]$ denote the truth value of the subquery rooted at node $v$, thus $T(v_?)$ denotes the truth value of the query. 
If root $v_?$ is merged from its child nodes $\{v_?^1, \dots, v_?^K\}$\footnote{Note that they are the same $v_?$ variable in the FOL query, but are separated to different nodes in the computation graph.} by intersection (type \Roman{1}) or union (type \Roman{2}), respectively, 
\begin{equation}
\begin{aligned}
    T(v_?) = \begin{cases}
    \top_{1\leq i\leq K}(T(v_?^i)), &\text{type \Roman{1}} \\
    \bot_{1\leq i\leq K}(T(v_?^i)), &\text{type \Roman{2}} \\
    \end{cases}
\end{aligned}
\label{eq:qt1}
\end{equation}
If the root $v_?$ is connected by a relational edge $r$ with its child node $c$ or $v_k$, then for relational projection (type \Roman{3}) and anti-relational projection (type \Roman{4}), respectively,
\begin{equation}
\begin{aligned}
    T(v_?) = \begin{cases}
    r(c, v_?)\ \text{or}\ T(v_k)\ \top\ r(v_k, v_?), &\text{type \Roman{3}} \\
    \lnot r(c, v_?)\ \text{or}\ T(v_k)\ \top\ \lnot r(v_k, v_?), &\text{type \Roman{4}} \\
    \end{cases}
\end{aligned}
\label{eq:qt2}
\end{equation}
The truth value of a query computation tree can be derived recursively according to Eq.~\ref{eq:qt1}, \ref{eq:qt2}. The optimization problem is therefore
\begin{align*}
    q[v_?]=v_?.\ v_1, \dots, v_N 
    = &\argmax_{v_1, \dots, v_N\in \mathcal{V}}\,T(v_?)
\end{align*}
\begin{restatable}{proposition}{prop}
The optimization problem defined on the query computation tree is equivalent to that defined on its FOL form.
\label{prop:eq}
\end{restatable}
We provide rigorous proof in Appendix~\ref{app:prop}.

\begin{figure*}[t]
\centering
\includegraphics[width=0.9\linewidth,trim=120 60 40 20,clip]{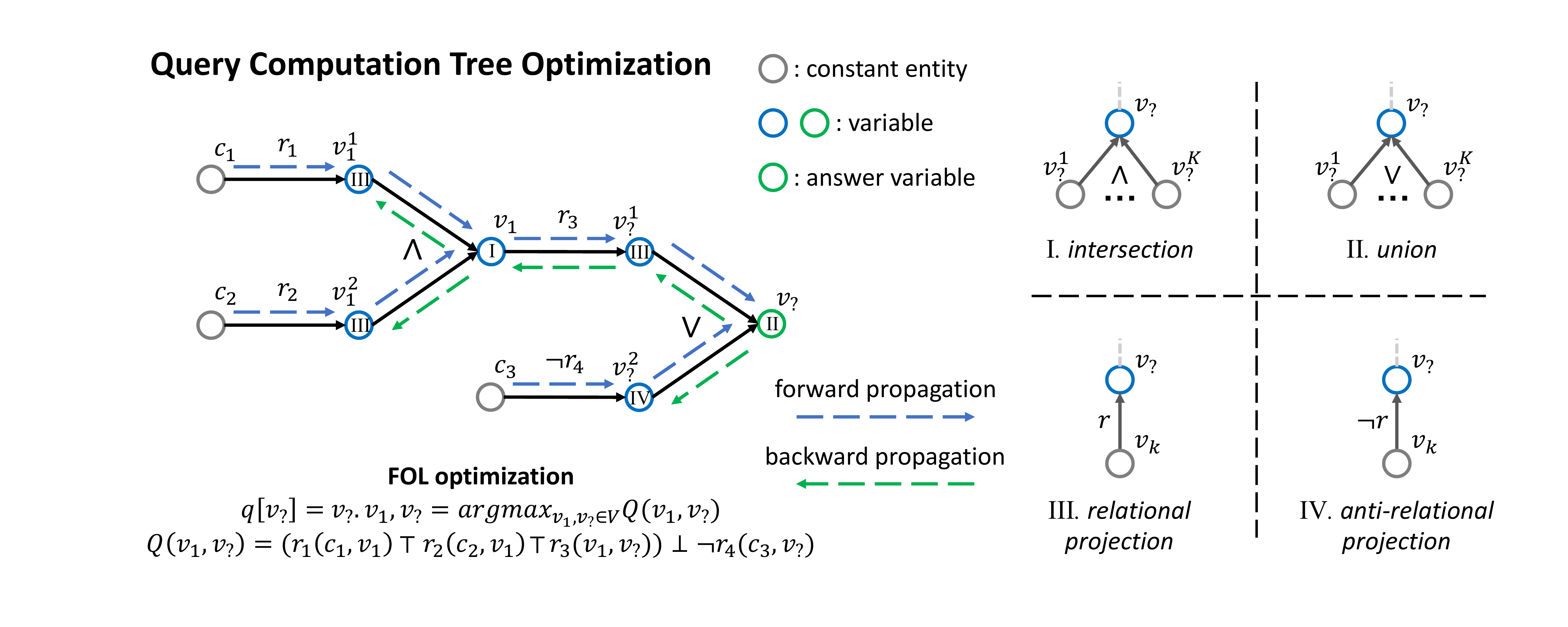}
\caption{Overview of QTO. For a FOL query optimization problem, we find the optimal solution by performing the forward/backward propagation on its tree-like computation graph, i.e., the query computation tree. Four types of local computation are involved, corresponding to the four types of local structure in the query computation tree.}
\label{fig:qto}
\end{figure*}

\subsection{Query Computation Tree Optimization}
We propose QTO (Query Computation Tree Optimization) that executes on the query computation tree to find the optimal set of entity assignments.
We show an overview of QTO in Figure~\ref{fig:qto}.

We define a \emph{neural adjacency matrix} $\mathbf{M}_r\in [0, 1]^{|\mathcal{V}|\times|\mathcal{V}|}$ for each relation $r$: $(\mathbf{M}_r)_{i, j}=r(e_i, e_j)$, where $e_i, e_j\in\mathcal{V}$ are the $i, j$-th entities.
Recall that a KGE model can score the likelihood of the atomic formula via $f_r(e_i, e_j)$, and we need to calibrate the score to a probability between $[0, 1]$.
The calibration function should be monotonically increasing, and if there is an edge $r$ between $e_i$ and $e_j$, the score should be faithfully calibrated to $1$.
Recall that for each triplet $(h, r, t)$, the KGE model is trained by maximizing its normalized probability $\exp(f_r(h, t))/\sum_{e\in\mathcal{V}}\exp(f_r(h, e))$.
Considering that there might be multiple valid tail entities for $(h, r)$, which should all obtain a probability close to 1, hence we multiply the normalized probability by the number of tail entities $N_t = \max\{1, |\{(e_i, r, e)\in\mathcal{E}|e\in\mathcal{V}\}|\}$. 
Note that we set the lower bound to 1 to allow positive scores for long tailed entity $e_i$ that does not have known edges of a certain type $r$ in the training set. Thus, we obtain
\begin{equation}
    \hat{r}(e_i, e_j) = \frac{\exp(f_r(e_i, e_j))\cdot N_t}{\sum_{e\in\mathcal{V}}\exp(f_r(e_i, e))}
\label{eq:norm}
\end{equation}
We further round $\hat{r}(e_i, e_j)$ so that it is between $[0, 1]$ and faithful to real triplets in the KG:
\begin{equation}
    r(e_i, e_j) = \begin{cases}
        1, &\text{if $(e_i, r, e_j)\in\mathcal{E}$} \\
        \min\{\hat{r}(e_i, e_j), 1-\delta\}, &\text{otherwise}
    \end{cases}
\label{eq:nam}
\end{equation}
where we set $\delta=0.0001>0$ to avoid over-confidence on predicted edges. This is also essential for the later corollary~\ref{cor} to hold. Also note that $\mathbf{M}_r$ degenerates to an adjacency matrix if all entries less than 1 are set to 0.
We discuss more ways to derive $\mathbf{M}$ in Appendix~\ref{app:discussion}.

\xhdr{Forward Propagation Procedure}
Let $T^*(v=e)$ denote the maximum truth value for the subquery rooted at node $v$ when $v$ is assigned an entity $e\in\mathcal{V}$.
Enumerate all possible entity assignments, we obtain $\mathbf{T}^*(v)=[T^*(v=e)]_{e\in \mathcal{V}}\in [0, 1]^{|\mathcal{V}|}$.
We maximize the truth value of the query computation tree by recursively deriving $\mathbf{T}^*(v)$ for each node $v$, and finally derive $\mathbf{T}^*(v_?)$ for root $v_?$.
Again, the recursive functions are categorized based on the 4 edge types between the root and its child node, according to Eq.~\ref{eq:qt1}, \ref{eq:qt2}.
Note that in this section, all products between vectors or matrices are performed element-wise, according to Hadamard product $\odot$.
If the root $v_?$ is merged from child nodes $\{v_?^1, \dots, v_?^K\}$ by intersection (type \Roman{1}), then
\begin{equation}
\begin{aligned}
    &T^*(v_?=e) = \top_{1\leq i\leq K}(T^*(v_?^i=e)) \\
    &\Rightarrow \mathbf{T}^*(v_?) = \prod_{1\leq i\leq K}\mathbf{T}^*(v_?^i)
\end{aligned}
\label{eq:f1}
\end{equation}
Similarly, if the root is merged by union (type \Roman{2}), we have
\begin{equation}
\begin{aligned}
    &T^*(v_?=e) = \bot_{1\leq i\leq K}(T^*(v_?^i=e)) \\
    &\Rightarrow \mathbf{T}^*(v_?) = \mathbf{1}-\prod_{1\leq i\leq K}(\mathbf{1}-\mathbf{T}^*(v_?^i))
\end{aligned}
\label{eq:f2}
\end{equation}
where $\textbf{1}\in \mathbb{R}^{|\mathcal{V}|}$ represents the vector of 1s.
Now we consider when the root is connected by a relational edge $r$ with its child $v_k$ (type \Roman{3}). Let $\max_i$ and $\max_j$ denote the max over row and column on a matrix, then the maximal term would be
\begin{equation}
\begin{aligned}
    &T^*(v_?=e) = \max_{e'\in\mathcal{V}}\{T^*(v_k=e')\ \top\ r(e', e)\} \\
    &\Rightarrow\mathbf{T}^*(v_?) = \max_j\left( 
    \begin{pmatrix} 
        \mathbf{T}^*(v_k)^T \cdots \scriptstyle\times |\mathcal{V}|
    \end{pmatrix}
    \odot\mathbf{M}_r\right)
\end{aligned}
\label{eq:f3}
\end{equation}
If the child node is a constant entity $c=e_i$, then
\begin{equation}
\begin{aligned}
    T^*(v_?=e) = r(c, e)
    \Rightarrow \mathbf{T}^*(v_?) = \text{row}_i(\mathbf{M}_r)
\end{aligned}
\label{eq:f3c}
\end{equation}
Similarly, we can derive the recursion for anti-relational projection (type \Roman{4}), here $\mathbf{1}$ denotes an all-one matrix:
\begin{equation}
\begin{aligned}
    &T^*(v_?=e) = \max_{e'\in\mathcal{V}}\{T^*(v_k=e')\ \top\ (1-r(e', e))\} \\
    &\Rightarrow 
    \mathbf{T}^*(v_?) = \max_j\left( 
    \begin{pmatrix} 
        \mathbf{T}^*(v_k)^T \cdots \scriptstyle\times |\mathcal{V}|
    \end{pmatrix}
    \odot(\mathbf{1}-\mathbf{M}_r)\right)
\end{aligned}
\label{eq:f4}
\end{equation}
If the child node is a constant entity $c=e_i$, we have
\begin{equation}
    T^*(v_?=e) = 1-r(c, e)
    \Rightarrow \mathbf{T}^*(v_?) = \text{row}_i(\mathbf{1}-\mathbf{M}_r)
\label{eq:f4c}
\end{equation}
In the forward propagation, the algorithm computes $\mathbf{T}^*(v)$ for each node $v$ from leaf to root in the query computation tree.
Hence, by definition, the solution to the optimization problem is:
\begin{equation}
\begin{aligned}
    &\max_{v_1, \dots, v_N\in\mathcal{V}}T(v_?) = \max\{\mathbf{T}^*(v_?)\} \\
    &q[v_?]= e_t: t=\argmax\{\mathbf{T}^*(v_?)\}
\end{aligned}
\label{eq:opt_root}
\end{equation}
where $\argmax$ on a vector returns the index of the largest element.

\xhdr{Backward Propagation Procedure}
Once a forward pass has calculated $\mathbf{T}^*(v)$ for every node $v$ in the tree, we can obtain the optimal set of entity assignments via a backward pass.
After obtaining the optimal assignment for root $v_?=e_t$ by Eq.~\ref{eq:opt_root}, we can obtain the optimal assignment for its child according to the four edge types. For type \Roman{1} and \Roman{2}, since the child nodes represent the same variable as their parent, we have
\begin{equation}
    v_?^i = e_t,\ i=1, 2,\dots, K
\label{eq:b12}
\end{equation}
While for type \Roman{3}, it can be deduced that
\begin{equation}
\begin{aligned}
    &v_k = \argmax_{e\in\mathcal{V}}\{T^{*}(v_k=e)\ \top\ r(e, e_t)\} \\
    &= e_i: i=\argmax\left(\mathbf{T}^*(v_k)^T\odot\text{col}_t(\mathbf{M}_r)\right)
\end{aligned}
\label{eq:b3}
\end{equation}
Similarly, for negation type \Roman{4}:
\begin{equation}
\begin{aligned}
    &v_k = \argmax_{e\in\mathcal{V}}\{T^{*}(v_k=e)\ \top\ (1-r(e, e_t))\} \\
    &= e_i: i=\argmax\left(\mathbf{T}^*(v_k)^T\odot(\mathbf{1}-\text{col}_t(\mathbf{M}_r))\right)
\end{aligned}
\label{eq:b4}
\end{equation}
Intuitively, the backward procedure finds the optimal entity assignments for intermediate variables from root back to leaf nodes.
Note that the backward computation can be executed for arbitrary assignments of the answer variable, thus providing interpretations for all potential answers.

According to the procedure described above, we show the algorithm of the forward-backward query computation tree optimization in Appendix~\ref{app:qto}.
Formally, we show that the procedure can find the optimal solution, and the full proof is in Appendix~\ref{app:thm}.
\begin{restatable}{theorem}{thm}
The forward procedure in Eq.~\ref{eq:f1}-\ref{eq:opt_root} can find the maximum value of $T(v_?)$, and the backward procedure in Eq.~\ref{eq:b12}-\ref{eq:b4} returns a set of assignments that obtains the optimal truth value.
\label{thm:main}
\end{restatable}

Also, for easy answers $\mathcal{V}_0$ (all links along the path for deducing the answer exist in the KG), the maximum value $T^*(v_?=e)=1, e\in \mathcal{V}_0$.
By the theorem, our QTO method can find such optimal value and return the set of assignments to obtain it.
As we set an $\delta>0$ in Eq.~\ref{eq:nam}, it holds that $T^*(v_?=e)<1, e\notin \mathcal{V}_0$.
Hence we have the following corollary:
\begin{corollary}
The procedure can always find the easy answers $\mathcal{V}_0\subseteq \mathcal{V}$ for a query, by finding entries of $\mathbf{T}^*(v_?)$ with value $1$.
\label{cor}
\end{corollary}
In summary, the forward procedure is sufficient for obtaining the answer entity, and the backward procedure is for assigning the intermediate variables to ensure interpretability.
Alternatively, our method can be viewed via a message passing aspect~\cite{pearl1982reverend,pearl2009causality}.
Think of the entity assignment for each variable as state, and relational projection as transition between states.
The query computation tree serves as an inference net where we aim to infer the most likely state of the root variable under the observations on leaf nodes.
From this perspective, we conclude that QTO reduces the combinatorial search space by utilizing the independence encoded in the tree structure, and can efficiently find the optimal solution with only local computations.

\subsection{Discussion}

\xhdr{Space Complexity}
We first consider the memory usage of QTO.
The neural adjacency matrix $\mathbf{M}$ contains $|\mathcal{R}|\cdot|\mathcal{V}|^2$ entries.
We notice that due to the sparsity of the KG, most of the entries in $\mathbf{M}$ have small values, and can be filtered by a threshold $\epsilon>0$ while maintaining precision.
By finding an appropriate $\epsilon$, $\mathbf{M}$ can be efficiently stored on a single GPU.
In Appendix~\ref{app:discussion}, we demonstrate that the filtering procedure significantly reduce the storage of the neural adjacency matrix, while faithfully preserving the majority of the salient information within the matrix.

\begin{table*}[t]
    \centering
    \begin{adjustbox}{width=\textwidth}
    \begin{tabular}{lccccccccccccccccc}
        \toprule
        \bf{Method} & \bf{avg$_p$} & \bf{avg$_{ood}$} & \bf{avg$_n$} & \bf{1p} & \bf{2p} & \bf{3p} & \bf{2i} & \bf{3i} & \bf{pi} & \bf{ip} & \bf{2u} & \bf{up} & \bf{2in} & \bf{3in} & \bf{inp} & \bf{pin} & \bf{pni} \\
        \midrule
        \multicolumn{18}{c}{\textbf{FB15k}} \\
        \midrule
        GQE & 28.0 & 20.1 & - & 54.6 & 15.3 & 10.8 & 39.7 & 51.4 & 27.6 & 19.1 & 22.1 & 11.6 & - & - & - & - & - \\
        Query2Box & 38.0 & 29.3 & - & 68.0 & 21.0 & 14.2 & 55.1 & 66.5 & 39.4 & 26.1 & 35.1 & 16.7 & - & - & - & - & - \\
        BetaE & 41.6 & 34.3 & 11.8 & 65.1 & 25.7 & 24.7 & 55.8 & 66.5 & 43.9 & 28.1 & 40.1 & 25.2 & 14.3 & 14.7 & 11.5 & 6.5 & 12.4 \\
        CQD-CO & 46.9 & 35.3 & - & 89.2 & 25.3 & 13.4 & 74.4 & 78.3 & 44.1 & 33.2 & 41.8 & 21.9 & - & - & - & - & - \\
        CQD-Beam & 58.2 & 49.8 & - & 89.2 & 54.3 & 28.6 & 74.4 & 78.3 & 58.2 & 67.7 & 42.4 & 30.9 & - & - & - & - & - \\
        ConE & 49.8 & 43.4 & 14.8 & 73.3 & 33.8 & 29.2 & 64.4 & 73.7 & 50.9 & 35.7 & 55.7 & 31.4 & 17.9 & 18.7 & 12.5 & 9.8 & 15.1 \\
        GNN-QE & 72.8 & 68.9 & 38.6 & 88.5 & \bf{69.3} & 58.7 & 79.7 & 83.5 & 69.9 & 70.4 & 74.1 & 61.0 & 44.7 & 41.7 & 42.0 & 30.1 & \bf{34.3} \\
        \midrule
        QTO & \bf{74.0} & \bf{71.8} & \bf{49.2} & \bf{89.5} & 67.4 & \bf{58.8} & \bf{80.3} & \bf{83.6} & \bf{75.2} & \bf{74.0} & \bf{76.7} & \bf{61.3} & \bf{61.1} & \bf{61.2} & \bf{47.6} & \bf{48.9} & 27.5 \\
        \midrule[0.08em]
        \multicolumn{18}{c}{\textbf{FB15k-237}} \\
        \midrule
        GQE & 16.3 & 10.3 & - & 35.0 & 7.2 & 5.3 & 23.3 & 34.6 & 16.5 & 10.7 & 8.2 & 5.7 & - & - & - & - & - \\
        Query2Box & 20.1 & 15.7 & - & 40.6 & 9.4 & 6.8 & 29.5 & 42.3 & 21.2 & 12.6 & 11.3 & 7.6 & - & - & - & - & - \\
        BetaE & 20.9 & 14.3 & 5.5 & 39.0 & 10.9 & 10.0 & 28.8 & 42.5 & 22.4 & 12.6 & 12.4 & 9.7 & 5.1 & 7.9 & 7.4 & 3.5 & 3.4 \\
        CQD-CO & 21.8 & 15.6 & - & 46.7 & 9.5 & 6.3 & 31.2 & 40.6 & 23.6 & 16.0 & 14.5 & 8.2 & - & - & - & - & - \\
        CQD-Beam & 22.3 & 15.7 & - & 46.7 & 11.6 & 8.0 & 31.2 & 40.6 & 21.2 & 18.7 & 14.6 & 8.4 & - & - & - & - & - \\
        FuzzQE & 24.0 & 17.4 & 7.8 & 42.8 & 12.9 & 10.3 & 33.3 & 46.9 & 26.9 & 17.8 & 14.6 & 10.3 & 8.5 & 11.6 & 7.8 & 5.2 & 5.8 \\
        ConE & 23.4 & 16.2 & 5.9 & 41.8 & 12.8 & 11.0 & 32.6 & 47.3 & 25.5 & 14.0 & 14.5 & 10.8 & 5.4 & 8.6 & 7.8 & 4.0 & 3.6 \\
        GNN-QE & 26.8 & 19.9 & 10.2 & 42.8 & 14.7 & 11.8 & 38.3 & 54.1 & 31.1 & 18.9 & 16.2 & 13.4 & 10.0 & 16.8 & 9.3 & 7.2 & \bf{7.8} \\
        \midrule
        QTO & \bf{33.5} & \bf{27.6} & \bf{15.5} & \bf{49.0} & \bf{21.4} & \bf{21.2} & \bf{43.1} & \bf{56.8} & \bf{38.1} & \bf{28.0} & \bf{22.7} & \bf{21.4} & \bf{16.8} & \bf{26.7} & \bf{15.1} & \bf{13.6} & 5.4 \\
        \midrule[0.08em]
        \multicolumn{18}{c}{\textbf{NELL995}} \\
        \midrule
        GQE & 18.6 & 12.5 & - & 32.8 & 11.9 & 9.6 & 27.5 & 35.2 & 18.4 & 14.4 & 8.5 & 8.8 & - & - & - & - & - \\
        Query2Box & 22.9 & 15.2 & - & 42.2 & 14.0 & 11.2 & 33.3 & 44.5 & 22.4 & 16.8 & 11.3 & 10.3 & - & - & - & - & - \\
        BetaE & 24.6 & 14.8 & 5.9 & 53.0 & 13.0 & 11.4 & 37.6 & 47.5 & 24.1 & 14.3 & 12.2 & 8.5 & 5.1 & 7.8 & 10.0 & 3.1 & 3.5 \\
        CQD-CO & 28.8 & 20.7 & - & 60.4 & 17.8 & 12.7 & 39.3 & 46.6 & 30.1 & 22.0 & 17.3 & 13.2 & - & - & - & - & - \\
        CQD-Beam & 28.6 & 19.8 & - & 60.4 & 20.6 & 11.6 & 39.3 & 46.6 & 25.4 & 23.9 & 17.5 & 12.2 & - & - & - & - & - \\
        FuzzQE & 27.0 & 18.4 & 7.8 & 47.4 & 17.2 & 14.6 & 39.5 & 49.2 & 26.2 & 20.6 & 15.3 & 12.6 & 7.8 & 9.8 & 11.1 & 4.9 & 5.5 \\
        ConE & 27.2 & 17.6 & 6.4 & 53.1 & 16.1 & 13.9 & 40.0 & 50.8 & 26.3 & 17.5 & 15.3 & 11.3 & 5.7 & 8.1 & 10.8 & 3.5 & 3.9 \\
        GNN-QE & 28.9 & 19.6 & 9.7 & 53.3 & 18.9 & 14.9 & 42.4 & \bf{52.5} & 30.8 & 18.9 & 15.9 & 12.6 & 9.9 & 14.6 & 11.4 & 6.3 & \bf{6.3} \\
        \midrule
        QTO & \bf{32.9} & \bf{24.0} & \bf{12.9} & \bf{60.7} & \bf{24.1} & \bf{21.6} & \bf{42.5} & 50.6 & \bf{31.3} & \bf{26.5} & \bf{20.4} & \bf{17.9} & \bf{13.8} & \bf{17.9} & \bf{16.9} & \bf{9.9} & 5.9 \\
        \bottomrule
    \end{tabular}
    \end{adjustbox}
    \caption{Test MRR results (\%) on complex query answering across all query types. avg$_p$ is the average on EPFO queries; avg$_{ood}$ is the average on out-of-distribution (OOD) queries; avg$_n$ is the average on queries with negation. Results on Hits@1 are in Appendix~\ref{app:h@1}.}
    \label{tb:main}
\end{table*}

\xhdr{Time Complexity}
We further consider the time complexity of our QTO method. During the forward pass, each variable is computed one time with complexity $O(|\mathcal{V}|^2)$, where the bottleneck is on (anti-)relational projection according to Eq.~\ref{eq:f3}, \ref{eq:f4}; during the backward pass, each variable is computed in $O(|\mathcal{V}|)$.
Notice that the KG is usually sparse, and $\mathbf{T}^*(v_k)$ is also a sparse vector (since most entries in $\mathbf{M}$ are filtered to 0), we can implement Eq.~\ref{eq:f3}, \ref{eq:f4} in a more efficient way. 
We take the nonzero entries of $\mathbf{T}^*(v_k)$ and multiply them with the corresponding rows in $\mathbf{M}_r$, obtaining a complexity of $O(|\mathcal{V}|\cdot |\mathbf{T}^*(v_k)>0|)$.
Hence, the total time complexity for a single query is $O(N'|\mathcal{V}|\cdot \max_k|\mathbf{T}^*(v_k)>0|)$, where $N'$ is the number of projections in the query.
In comparison, GNN-QE~\cite{zhu2022neural} has $O(N'(|\mathcal{V}|d^2+|\mathcal{E}|d))$ complexity ($d$ is the embedding dimension), and CQD-beam~\cite{arakelyan2021complex} has $O(N'|\mathcal{V}|bd)$ complexity ($b$ is the beam size).
Moreover, our computational steps only involve matrix operations, which can be effectively accelerated on GPUs.
As we show in Sec.~\ref{sec:efficiency}, QTO achieves faster speed when compared with previous methods in real world KG.
Meanwhile, the pre-computing of the matrix $\mathbf{M}$ can be efficiently parallelized on GPUs.
We combine more discussions on the scalability of QTO in Appendix~\ref{app:discussion}.

%% file: 040experiments.tex
\section{Experiments}
\label{sec:experiments}

\subsection{Experimental Setup}
\xhdr{Datasets}
We experiment on three knowledge graph datasets, including FB15k~\cite{bordes2013translating}, FB15k-237~\cite{toutanova2015observed}, NELL995~\cite{xiong2017deeppath}.
Detailed statistics of the datasets are listed in Appendix~\ref{app:stat}.
We use the standard FOL queries generated in~\cite{ren2019query2box,ren2020beta}, consisting of 9 types of EPFO queries (1p/2p/3p/2i/3i/pi/ip/2u/up) and 5 types of queries with negation (2in/3in/inp/pin/pni).
Specifically, `p', `i', and `u' stand for `projection', `intersection', and `union' in the query structure (we show a detailed statistic of the query types in Appendix~\ref{app:query}).

\xhdr{Baselines}
We compare QTO against state-of-the-art methods on complex query answering, including GQE~\cite{hamilton2018embedding}, Query2Box~\cite{ren2019query2box}, BetaE~\cite{ren2020beta}, CQD-CO~\cite{arakelyan2021complex}, CQD-beam~\cite{arakelyan2021complex}, ConE~\cite{zhang2021cone}, FuzzQE~\cite{chen2022fuzzy}, and GNN-QE~\cite{zhu2022neural}.

\begin{figure*}[t]
\centering
\subfigure{
\begin{minipage}{0.33\linewidth}
\centering
\includegraphics[width=2.2in]{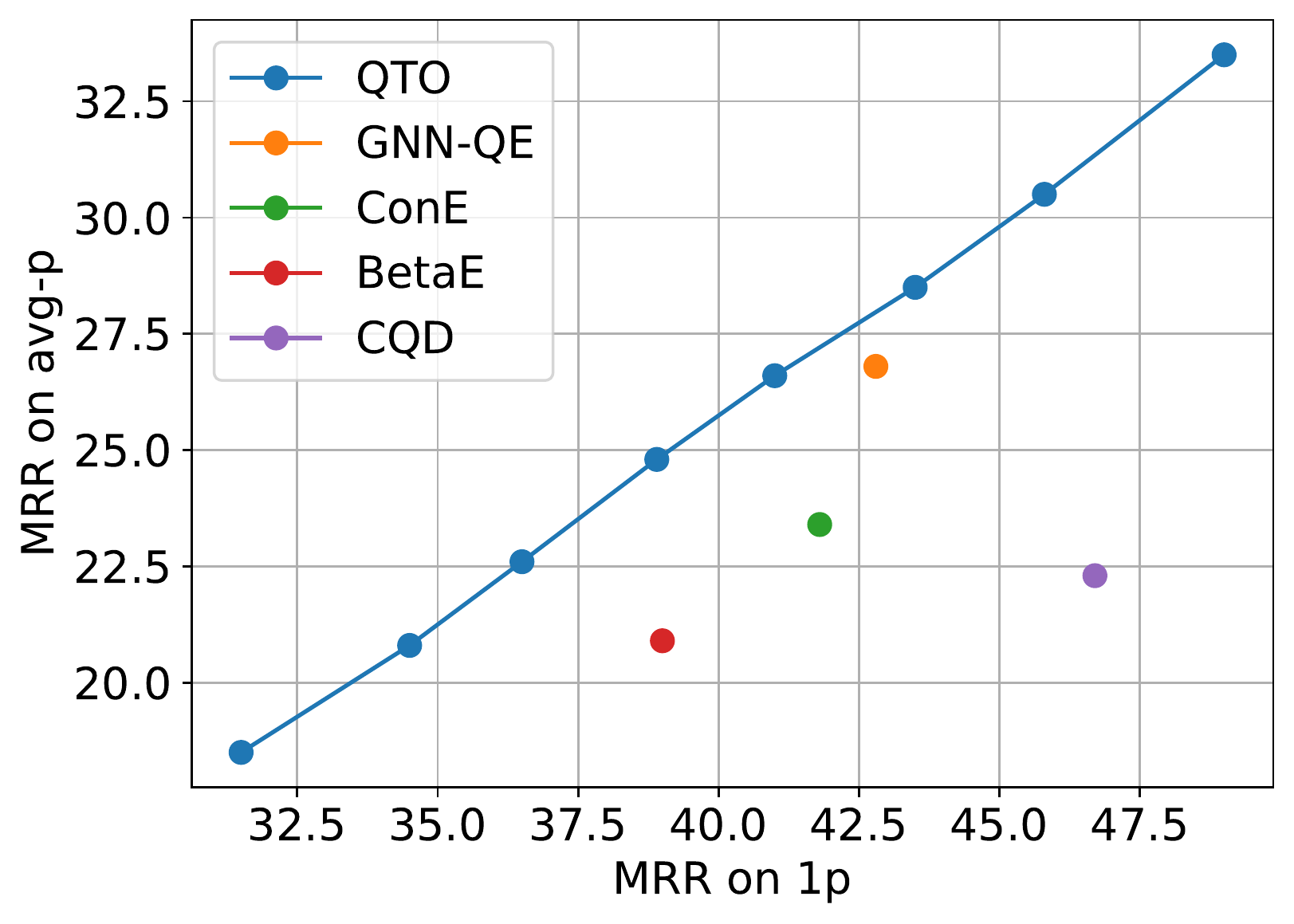}
\label{fig:p-1p}
\end{minipage}%
}%
\subfigure{
\begin{minipage}{0.33\linewidth}
\centering
\includegraphics[width=2.2in]{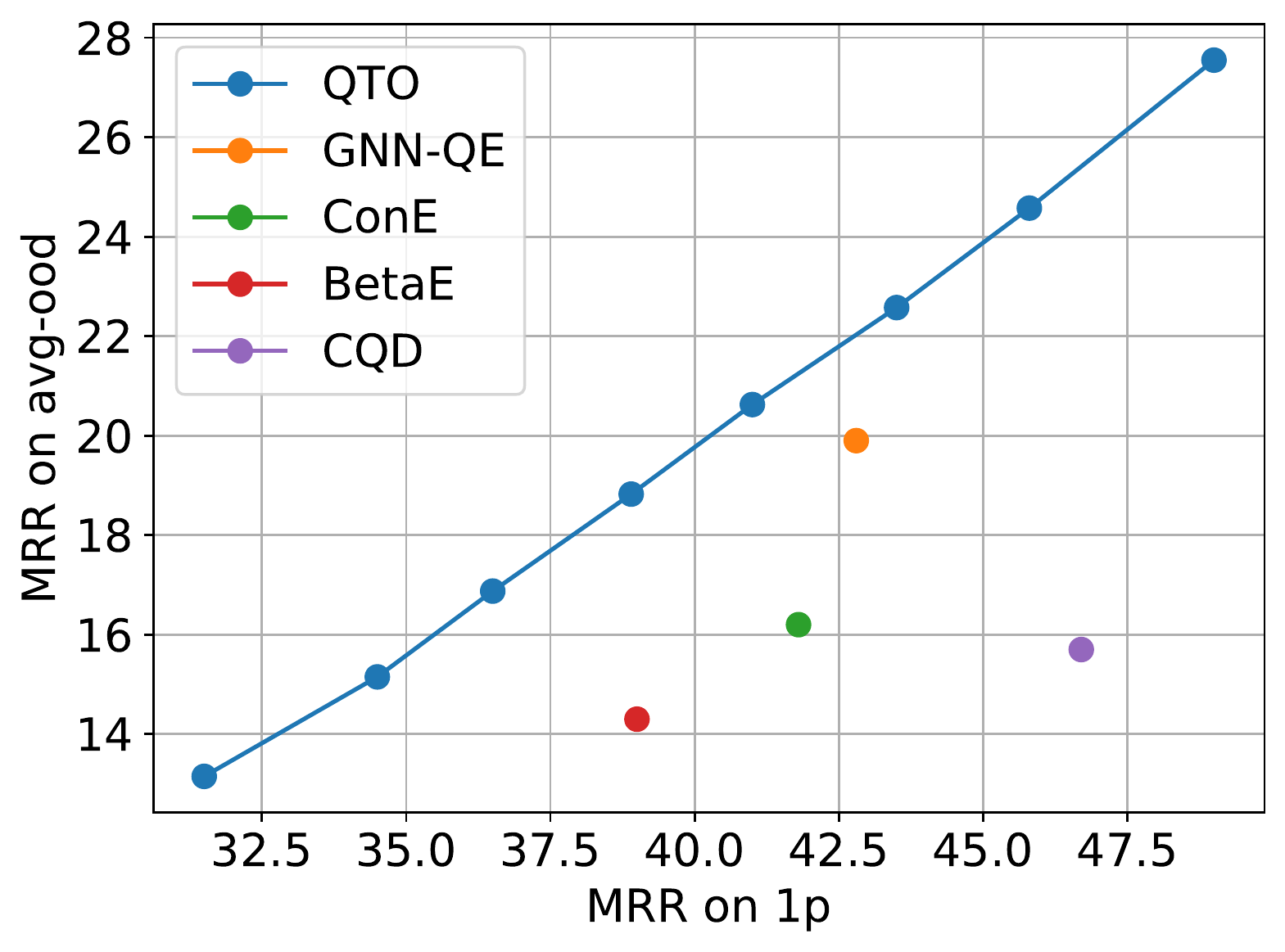}
\label{fig:ood-1p}
\end{minipage}%
}%
\subfigure{
\begin{minipage}{0.33\linewidth}
\centering
\includegraphics[width=2.2in]{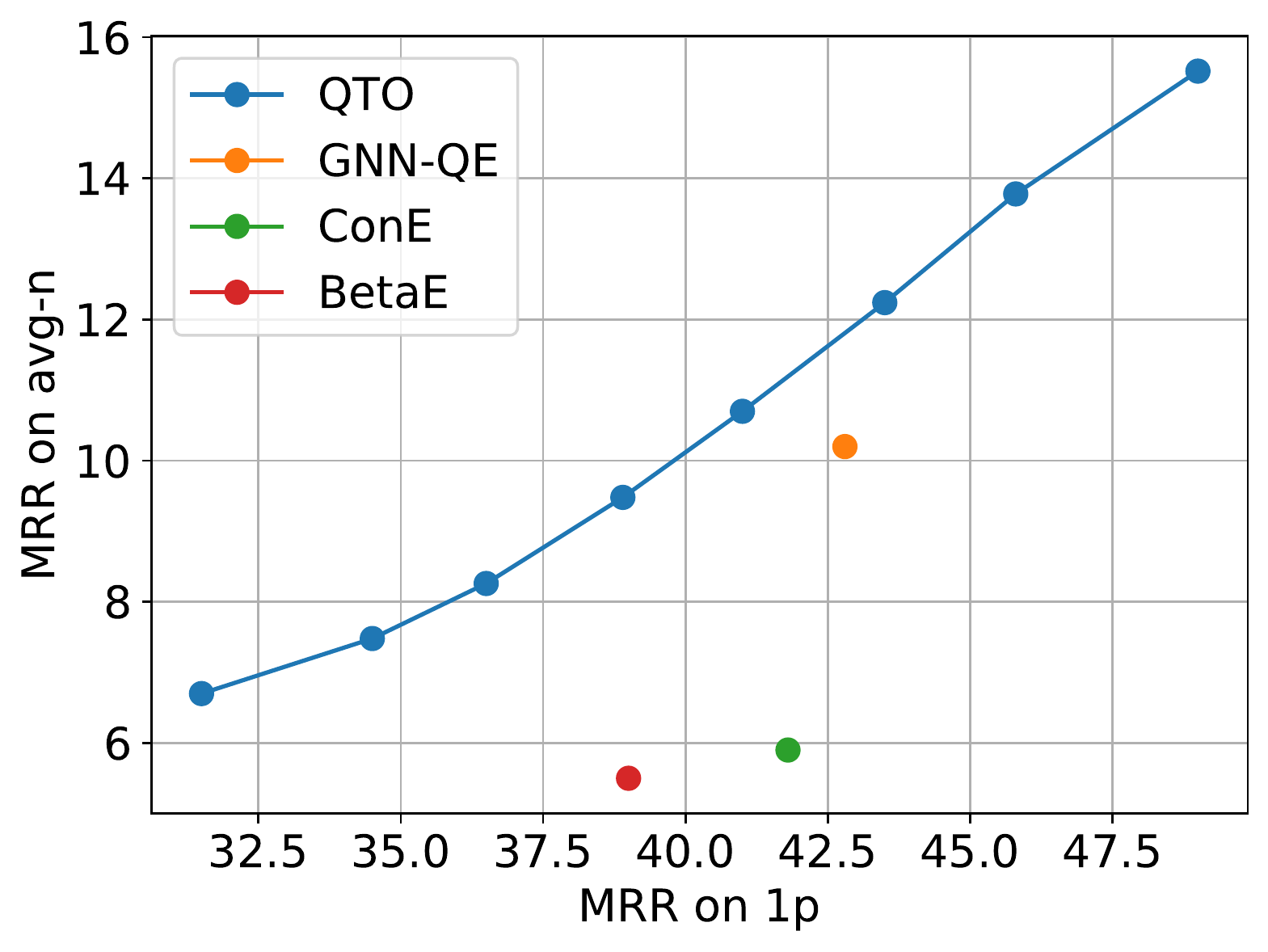}
\label{fig:n-1p}
\end{minipage}%
}%
\centering
\vspace{-7mm}
\caption{Test MRR on avg$_p$, avg$_{ood}$, and avg$_n$ w.r.t. MRR on 1p (one-hop) queries, evaluated on FB15k-237.}
\label{fig:avg-1p}
\end{figure*}

\xhdr{Evaluation Protocol}
For each complex query, its answers are divided into \emph{easy answers} and \emph{hard answers}, based on whether the answer can be derived by existing edges in the graph~\cite{ren2019query2box}.
Specifically, in valid/test set, the easy answers are the entities that can be inferred by edges in training/valid graph, while hard answers are those that can be inferred by predicting missing edges in valid/test graph.
We calculate standard evaluation metrics including mean reciprocal rank (MRR) and Hits at K (Hits@K) on hard answers, in the filtered setting where all easy and hard answers are filtered out during ranking.
QTO ranks the final answer entity according to $\mathbf{T}^*(v_?)$.

\xhdr{Implementation Details}
To answer queries on each of the KG, we first train a KGE model on its training graph.
We use ComplEx~\cite{trouillon2016complex} trained with N3 regularizor~\cite{lacroix2018canonical} and auxiliary relation prediction task~\cite{chen2021relation}.
Then the neural adjacency matrix $\textbf{M}$ is calculated based on the scores given by the KGE model, and calibrated according to Eq.~\ref{eq:norm}, \ref{eq:nam}.
To save the memory usage of $\textbf{M}$, we find an appropriate $\epsilon$ such that after filtering all entries $<\epsilon$, the sparse matrix can be stored on a single GPU.
We further observe that the entries in the neural adjacency matrix are small (due to the Softmax operation over all entities), and QTO may fail to filter the tail entities (reached by relational edges) for the anti-relational atoms in queries with negation.
Hence we scale the matrix by $\alpha$ on queries with negation: $\textbf{M}=\min\{\textbf{1}, \alpha\cdot\textbf{M}\}$, and search the best $\alpha$ among integers in $[1, 10]$ based on the performance on valid queries. We report the implementation details and hyperparameters in Appendix~\ref{app:detail}.

\subsection{Results on Complex Query Answering}
\label{sec:result}
Table~\ref{tb:main} reports the complex query answering results on the three datasets.
Previous baselines, except CQD, are trained on 1p/2p/3p/2i/3i queries, hence the other 4 types of EPFO queries serve as OOD queries, and we report the average result on these queries in avg$_{ood}$.
We observe that QTO significantly outperforms baseline methods across all datasets.
Notably, even without additional training on complex queries, QTO yields a relative gain of 13.5\%, 21.8\%, and 37.5\% on avg$_{p}$, avg$_{ood}$, and avg$_{n}$, compared to previous Sota method GNN-QE.
This suggests that our method has better reasoning skills and can generalize to more query structures.
We further observe that QTO outperforms CQD by a large margin, suggesting the effectiveness of our exact optimization method.

Moreover, we show that QTO generalizes better from one-hop answering, and its performance can benefit from a more powerful one-hop answering KGE model.
Figure~\ref{fig:avg-1p} plots the performance comparison on complex queries w.r.t. the performance on 1p (one-hop) queries, on FB15k-237 (we also provide the plots for each type of complex queries in Appendix~\ref{app:plot})\footnote{We vary the one-hop query answering model by taking the ComplEx KGE model at different epochs (1, 2, 4, 10, 25, 70, 250).}.
We make three key observations: (a) The curve of QTO lies on the top left to baseline methods, suggesting it generalizes better from one-hop answering to complex queries;
(b) Such a gap between QTO and CQD is significant, indicating our optimization method better leverages the KGE model;
(c) Previous trainable methods suffer from the coupling of one-hop queries with multi-hop queries during training, preventing them from getting higher accuracies.
In contrast, QTO disentangles them and can be directly benefited from better one-hop models.

Additionally, we investigate the effect of the hyperparameters $\epsilon, \alpha$ in QTO. 
We discover that the threshold $\epsilon$ poses a trade-off between memory usage and model accuracy.
As we decrease the value of $\epsilon$, we observe a corresponding increase in model accuracy. However, this improvement in accuracy comes at the expense of a rise in memory consumption, particularly within the neural adjacency matrix.
Further, the negation scaling coefficient $\epsilon$ has an effect on queries with negation.
We provide the detailed results in Table~\ref{tb:eps}, \ref{tb:alpha}. 

To provide more insights on QTO's extensibility and facilitate a more equitable comparison with prior methods, we conduct a series of experiments employing QTO alongside various KGE models, including TransE~\cite{bordes2013translating}, RotatE~\cite{sun2019rotate}, ComplEx+N3~\cite{lacroix2018canonical} without the auxiliary relation prediction task, and NBFNet~\cite{zhu2021neural}.\footnote{ComplEx+N3 and NBFNet are the KGE backbones of CQD and GNN-QE, respectively.}
The empirical results (shown in Table~\ref{tb:kge}) implies that the efficacy of QTO can be amplified by utilizing a more powerful one-hop answering KGE model.
Intriguingly, even under the same KGE backbone, QTO outperforms CQD and GNN-QE, suggesting the effectiveness of our optimization method.

\subsection{Interpretability Study}
As Thm~\ref{thm:main} suggests, our backward propagation procedure can find the most likely entity assignments on intermediate variables for any answer variable assignment $v_?=e$, which serves as an explanation for how the answer is derived.
In comparison, CQD-beam can only explain its predicted answer (number bounded by beam size), but cannot explain for an arbitrary answer.
We measure the interpretability of the method by the accuracy of its explanations for the hard answers.
Specifically, for each set of entity assignments, we can check whether the assignments are valid according to the full graph, i.e., whether the FOL expression under such valuation is true.\footnote{The assignments for 1p/2i/3i/2u/2in/3in queries are trivially true when the answer is true, hence we evaluate on other types.}
Table~\ref{tb:interpretability} reports the accuracy rate of QTO on Hits@K answers that the model predicts, as well as on all true answers.
Specifically, the accuracy corresponding to Hits@K is calculated by determining the mean accuracy of the interpretations, whilst assigning Hits@K predictions (true answers predicted within top-K) to the answer variable. 
The ``All'' category in the table represents the mean accuracy on all true answers, equivalent to Hits@$|\mathcal{V}|$.
The results suggest that accuracy improves for smaller values of K. 
This observation stems from the fact that higher-ranking predictions receive higher scores, thereby reflecting a greater degree of confidence in their intermediate variable interpretations.
Notably, we find that QTO can provide valid explanations for over 90\% of the Hits@1 answers it predicts.
Moreover, by observing the interpretation of QTO's answer prediction, we can analyze the failure behavior and open the black box behind complex query answering.
We show case studies of QTO's interpretation in Appendix~\ref{app:case}.

\begin{table}[t]
    \centering
    \begin{adjustbox}{width=0.48\textwidth}
    \begin{tabular}{lcccccccc}
        \toprule
        & \bf{2p} & \bf{3p} & \bf{pi} & \bf{ip} & \bf{up} & \bf{inp} & \bf{pin} & \bf{pni} \\
        \midrule
        on Hits@1 & 88.6 & 85.1 & 93.9 & 91.3 & 90.8 & 81.9 & 90.3 & 93.5\\
        on Hits@3 & 83.7 & 79.0 & 92.5 & 88.4 & 85.7 & 76.0 & 87.0 & 94.4 \\
        on Hits@10 & 77.8 & 72.3 & 90.8 & 85.7 & 79.1 & 71.0 & 82.5 & 95.4 \\
        on All & 65.7 & 56.7 & 84.3 & 78.7 & 64.8 & 52.7 & 68.5 & 94.3 \\
        \bottomrule
    \end{tabular}
    \end{adjustbox}
    \caption{Accuracy (\%) on intermediate variable interpretation given the Hits@K answers that QTO predicts and all true answers, evaluated on FB15k-237. Results on other datasets are in Appendix~\ref{app:inter}.}
    \label{tb:interpretability}
\end{table}

\begin{table}[t]
    \centering
    \begin{adjustbox}{width=0.48\textwidth}
    \begin{tabular}{lcccccccccc}
        \toprule
        \bf{Method} & \bf{avg} & \bf{1p} & \bf{2p} & \bf{3p} & \bf{2i} & \bf{3i} & \bf{pi} & \bf{ip} & \bf{2u} & \bf{up} \\
        \midrule
        GNN-QE & 36.5 & 40.9 & 23.6 & 27.4 & 34.8 & 53.4 & 39.9 & 60.0 & 27.8 & 20.3 \\
        QTO & \bf{28.2} & \bf{34.0} & \bf{20.9} & \bf{21.2} & \bf{33.4} & \bf{50.0} & \bf{33.8} & \bf{32.1} & \bf{12.9} & \bf{15.7} \\
        \bottomrule
    \end{tabular}
    \end{adjustbox}
    \caption{MAPE (\%, $\downarrow$) on answer set cardinality prediction, evaluated on FB15k-237. Results on other datasets are in Appendix~\ref{app:card}.}
    \label{tb:mape}
\end{table}
\vspace{-2mm}

\begin{table*}[t]
    \centering
    \begin{adjustbox}{width=0.95\textwidth}
    \begin{tabular}{lcccccccccccccccc}
    \toprule
    \bf{Method} & \bf{avg$_p$} & \bf{avg$_n$} & \bf{1p} & \bf{2p} & \bf{3p} & \bf{2i} & \bf{3i} & \bf{pi} & \bf{ip} & \bf{2u} & \bf{up} & \bf{2in} & \bf{3in} & \bf{inp} & \bf{pin} & \bf{pni} \\
    \midrule
    CQD & .457 & - & .782 & .534 & .378 & .549 & .473 & .338 & .212 & .604 & .243 & - & - & - & - & - \\
    GNN-QE & .878 & .984 & .716 & .974 & .973 & .848 & .620 & .874 & .902 & .999 & .994 & .995 & .988 & .964 & .977 & .996 \\
    QTO & 1.00 & 1.00 & 1.00 & 1.00 & 1.00 & 1.00 & 1.00 & 1.00 & 1.00 & 1.00 & 1.00 & 1.00 & 1.00 & 1.00 & 1.00 & 1.00 \\
    \toprule
    \end{tabular}
    \end{adjustbox}
    \caption{Valid Hits@1 results (\%) on \emph{easy} complex query answering across all query types, evaluated on FB15k-237.}
    \label{tb:easy}
\end{table*}

\begin{table*}[t]
    \centering
    \begin{adjustbox}{width=0.75\textwidth}
    \begin{tabular}{lccccccccccccc}
    \toprule
    \bf{Method}  & \bf{2p} & \bf{3p} & \bf{2i} & \bf{3i} & \bf{pi} & \bf{ip} & \bf{2u} & \bf{up} & \bf{2in} & \bf{3in} & \bf{inp} & \bf{pin} & \bf{pni} \\
    \midrule
    CQD & 24.8 & 37.6 & 12.4 & 9.6 & 41.2 & 84.8 & 9.6 & 25.4 & - & - & - & - & - \\
    GNN-QE & 9.0 & 13.6 & 9.2 & 14.4 & 13.4 & 13.6 & 9.2 & 13.4 & 9.2 & 14.0 & 13.8 & 13.8 & 13.8 \\
    QTO & 7.6 & 13.6 & 3.6 & 4.2 & 8.4 & 7.8 & 3.8 & 8.8 & 3.6 & 4.6 & 9.2 & 8.8 & 10.6 \\
    \toprule
    \end{tabular}
    \end{adjustbox}
    \caption{Inference time (\emph{ms}/query) on each type of query on FB15k-237, evaluated on one RTX 3090 GPU.}
    \label{tb:time}
\end{table*}

\subsection{Predicting the Cardinality of Answer Sets}

Predicting the cardinality of answer sets is also a practical yet challenging task.
For example, consider the question ``\emph{How many phisicists won the 1921 Nobel prize?}'', which directly asks for the number of entity assignments for the answer variable.
Our method can predict the cardinality of the answer set by $|\mathbf{1}[\mathbf{T}^*(v_?)>T]|$ for some threshold $T\in[0, 1]$.
We select the best threshold $T\in\{0.1, 0.2, \dots, 0.9\}$ according to the best validation accuracy.
Table~\ref{tb:mape} reports the mean absolute percentage error (MAPE) between QTO's cardinality prediction and the groundtruth on FB15k-237.
We only compare with GNN-QE since it is the only previous method that can predict the cardinality without explicit supervision~\cite{zhu2022neural}.
As shown, QTO outperforms GNN-QE by a large margin on cardinality prediction.

\begin{figure}[t]
\centering
\includegraphics[width=0.8\linewidth,trim=0 0 0 0,clip]{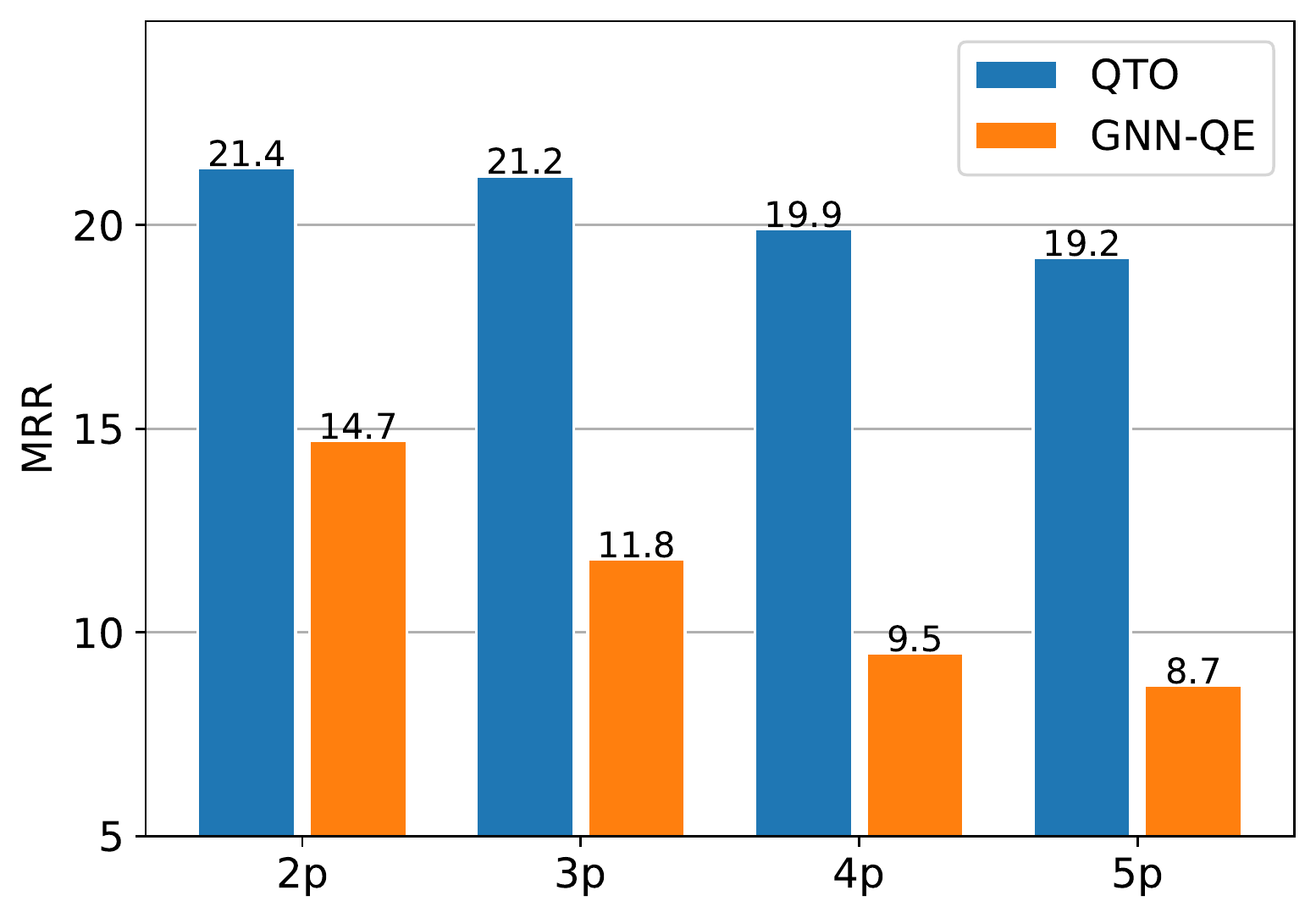}
\vspace{-2mm}
\caption{MRR on queries of 2, 3, 4, and 5 hops on FB15k-237.}
\label{fig:hop}
\end{figure}

\subsection{Analyses on Reasoning Skills and Efficiency}
\label{sec:efficiency}
We conduct a more comprehensive evaluation on QTO's reasoning skills on longer queries and easy queries.
We also involve an efficiency comparison between QTO and baseline methods.

\xhdr{Reasoning on Longer Queries}
Our main experiment shows the performance of QTO on 1, 2, and 3 hop queries. A natural question would be, can our method generalize to queries with longer hops?
We show the performance comparison between QTO and the previous Sota method GNN-QE on 2-5 hop queries, i.e., 2p/3p/4p/5p queries, in Figure~\ref{fig:hop}.
We see that the performance of GNN-QE significantly reduces with the number of hops, while QTO is more resistant to the increasing length of the reasoning chain.

\xhdr{Reasoning on Easy Queries}
We stated in Corollary~\ref{cor} that our method can always find the easy answers, thus they are all at Hits@1 in QTO's prediction.
Intuitively, QTO automatically degenerates to perform as a traditional searching algorithm if we set $\mathbf{M}$ as the adjacency matrix, while introducing uncertainty into $\mathbf{M}$ allows the method to perform probabilistic reasoning.
However, this is not the case for previous methods, as we report their Hits@1 on easy answers in Table~\ref{tb:easy}.
The results indicate that trainable method such as GNN-QE does not memorize the existing edges in the KG well, while optimization-based CQD losses accuracy even on easy answers to the approximation.
In contrast, QTO can faithfully answer complex queries based on existing edges, due to our theoretical guarantee.

\xhdr{Reasoning efficiency}
One common concern for QTO is its efficiency, since it involves (sparse) matrix products between matrices with sizes of $|\mathcal{V}|\times|\mathcal{V}|$.
We report the inference time on each type of query in Table~\ref{tb:time}.
As we can see, benefiting from GPU acceleration on matrix product, QTO achieves even higher efficiency compared to previous methods.

%% file: 050conclusion.tex
\section{Conclusion}
\label{sec:conclusion}

This paper proposes a novel optimization-based method QTO (Query Computation Tree Optimization) for answering complex logical queries on knowledge graphs.
QTO utilizes a pretrained KGE link predictor to score each one-hop atom, and can efficiently find an optimal set of entity assignment that maximizes the truth value of the complex query.
In particular, QTO optimizes directly on the tree-like computation graph, and searches for the optimal solution by a forward-backward propagation procedure.
Experiments show that QTO significantly outperforms previous methods.
Moreover, QTO is the first neural method that can explicit interpret the intermediate variables for any answer, and is faithful to the existing relationships in the graph.

%% file: 060appendix.tex
\appendix
\onecolumn

\section{Conversion Between FOL Expression and Query Computation Tree}
\label{app:conversion}
The conversion from a FOL expression (disjunctive normal form) to its query computation tree involves three steps: dependency graph generation, union branches merging, and variable separation.

\xhdr{Dependency Graph Generation}
Given a FOL expression, we first assign a unique node to each of the variables, and a unique node to the constant entity in each of the one-hop atoms. Note that there might be several nodes that correspond to the same constant entity, since they appear in different one-hop atoms.
Then we use undirected edges to connect the nodes according to the one-hop atoms.
Specifically, if $e^i_j=r(v', v)$ (or $r(c, v)$), then we connect the nodes of $v'$ (or $c$) and $v$ by an edge $r_i$.
Similarly, if $e^i_j=\lnot r(v', v)$ (or $\lnot r(c, v)$), then we connect the nodes of $v'$ (or $c$) and $v$ by an edge $\lnot r_i$.
The notation $i$ here will be used to distinguish the edges from different conjunctions.
The constructed undirected dependency multigraph has to be a tree, in other words, it is a connected acyclic graph.
We take the node $v_?$ as root, and assign directions for the edges such that they all point from child nodes to their parent nodes, during which we have to handle the inverse of relations.
The constant entities are naturally leaf nodes in the tree, since each entity node is only connected to one variable node.

\xhdr{Union Branches Merging}
Then we handle the union structures in the query computation tree.
On the path $\tau$ from root to every leaf node, if exists, we find the first node $v_i$ such that the edges between $v_i$ and its child node $v_j$ are all of the same relation, but in different conjunctions: $r_{t_1}, r_{t_2}, \dots, r_{t_p}$.
We merge these edges into a single edge $r_{t_1, t_2, \dots t_p}$, since they all correspond to the same one-hop atom but in different conjunctions, they can be merged by the distributive law:
\begin{equation}
    (P\land Q)\lor(P\land R) \Leftrightarrow P\land(Q\lor R)
\end{equation}
We assert that there is a subpath from $v_i$ to some $v_k$ within the path $\tau$ that only consists of edges $r_{t_1, t_2, \dots t_p}$, and $v_k$ is connected to different child nodes by relations from conjunctions $t_1, t_2, \dots t_p$.
We mark these edges as union, while the rest of one-to-many structures are marked as intersection.
Now the multigraph becomes a simple graph (no multiple edges).

\xhdr{Variable Separation}
For the one-to-many intersection/union structures in the tree, we separate the parent node $v_k$ into $v^1_k, v^2_k, \dots$ for each of the child nodes, and connect the $i$th child node with $v^i_k$ by its original relational edge, while all $v^1_k, v^2_k, \dots$ connect to $v_k$ by intersection/union edges.
Note that the union branches merging step may create one-to-many structures that consist of both intersection and union edges, take Figure~\ref{fig:conversion} for an example.
This can be taken care of by first separating $v_k$ into a union structure ($v^3_1$ and $v^4_1$ in the example), and then separating the child node into an intersection structure ($v^1_1$ and $v^2_1$ in the example).

We show an example of converting a FOL expression to its query computation tree in Figure~\ref{fig:conversion}. One can verify that the logic expression represented by the query computation tree is equivalent to the FOL expression.

\begin{figure}[htbp]
\centering
\includegraphics[width=1\linewidth,trim=0 60 0 20,clip]{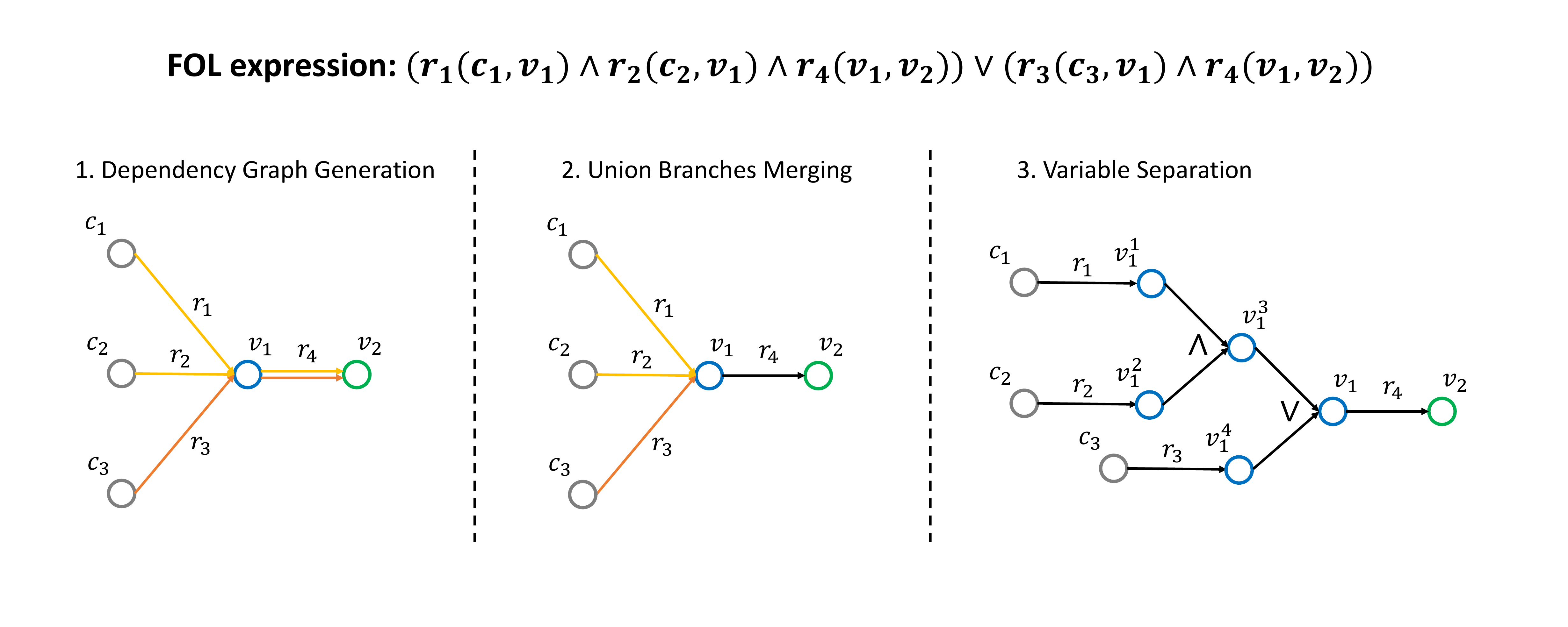}
\caption{Conversion from a FOL expression to its query computation tree.}
\label{fig:conversion}
\end{figure}

Moreover, we believe that \textbf{query computation tree actually serves as a more appropriate form to represent complex logical questions on KGs}.
Although the FOL expression is more general, most kinds of the FOL expressions are rarely asked in practice.
In fact, query computation tree is more in line with human's multi-hop questioning style --- where an intermediate variable is usually described by its relationships with other entities or their logical combinations, which corresponds to a subtree rooted at the variable in the query computation tree.

\section{Proof of \cref{prop:eq}: Equivalence of Optimization on the Query Computation Tree}
\label{app:prop}

\prop*
\begin{proof}
    Recall that $T(v_?)$ denotes the truth value of the query computation tree rooted at $v_?$, thus to prove the equivalence between the two optimization problems, it suffices to show that $T(v_?)=T(Q(v_?))$ where $T(q(v_?))$ is the truth value of the FOL expression corresponding to the query computation tree rooted at $v_?$.
    We prove it by induction. Assume the equation holds for the child nodes of $v_?$, then for $v_?$, consider the 4 types of edges connecting it to its child nodes.
    For type \Roman{1} (intersection), then according to Eq.~\ref{eq:qt1}, it holds that
    \begin{equation}
        Q(v_?) = \land_{1\leq i\leq K}Q(v^i_?)\ \Rightarrow\ 
        T(Q(v_?)) = \top_{1\leq i\leq K}T(v^i_?) = T(v_?)
    \end{equation}
    For type \Roman{2} (union), according to Eq.~\ref{eq:qt1}, we also have
    \begin{equation}
        Q(v_?) = \lor_{1\leq i\leq K}Q(v^i_?)\ \Rightarrow\ 
        T(Q(v_?)) = \bot_{1\leq i\leq K}T(v^i_?) = T(v_?)
    \end{equation}
    For type \Roman{3} (relational projection), according to Eq.~\ref{eq:qt2},
    \begin{equation}
        Q(v_?) = r(c, v_?)\text{ or }Q(v_k)\land r(v_k, v_?) \ \Rightarrow\ 
        T(Q(v_?)) = r(c, v_?)\text{ or }T(Q(v_k))\ \top\ r(v_k, v_?) = T(v_?)
    \label{eq:app3}
    \end{equation}
    Similarly, for type \Roman{4} (anti-relational projection), according to Eq.~\ref{eq:qt2},
    \begin{equation}
        Q(v_?) = \lnot r(c, v_?)\text{ or }Q(v_k)\land \lnot r(v_k, v_?) \ \Rightarrow\ 
        T(Q(v_?)) = \lnot r(c, v_?)\text{ or }T(Q(v_k))\ \top\ \lnot r(v_k, v_?) = T(v_?)
    \label{eq:app4}
    \end{equation}
The trivial case is when the child node of $v_?$ is a constant entity, and it also holds true according to Eq.~\ref{eq:app3}, \ref{eq:app4}.
Therefore, $T(v_?)=T(Q(v_?))$ is always true, indicating that the optimization on query computation tree is equivalent to the optimization on its FOL expression.
\end{proof}

\section{Query Computation Tree Optimization Algorithm}
\label{app:qto}

\begin{minipage}[!h]{\textwidth}
\begin{minipage}{.48\textwidth}
\centering
\begin{algorithm}[H]
\small
\caption{Forward Propagation Function}
\label{alg:forward}
\begin{algorithmic}
   \Function{Forward}{$v$, $M$, $T$}
   \State $V\leftarrow 1^{|\mathcal{V}|}$
   \If {$v$.type = \emph{intersection}}
   \For {each node $u$ in $v$.child}
   \State $U, T\leftarrow$ \Call{Forward}{$u$, $M$, $T$}
   \State $V\leftarrow V * U$
   \EndFor
   \EndIf
   \If {$v$.type = \emph{union}}
   \For {each node $u$ in $v$.child}
   \State $U, T\leftarrow$ \Call{Forward}{$u$, $M$, $T$}
   \State $V\leftarrow V * (1^{|\mathcal{V}|}-U)$
   \EndFor
   \State $V\leftarrow 1^{|\mathcal{V}|}-V$
   \EndIf
   \If {$v$.type = \emph{relational}}
   \If {$v$.child is a constant entity $i$}
   \State $V\leftarrow M[v.r][i][:]$
   \Else
   \State $U, T\leftarrow$ \Call{Forward}{$v$.child, $M$, $T$}
   \State $V\leftarrow \max_j\{U^T * M[v.r][:][:]\}$
   \EndIf
   \EndIf
   \If {$v$.type = \emph{anti-relational}}
   \If {$v$.child is a constant entity $i$}
   \State $V\leftarrow 1^{\mathcal{V}}-M[v.r][i][:]$
   \Else
   \State $U, T\leftarrow$ \Call{Forward}{$v$.child, $M$, $T$}
   \State $V\leftarrow \max_j\{U^T * (1^{\mathcal{V}\times\mathcal{V}} - M[v.r][:][:])\}$
   \EndIf
   \EndIf
   \State $T(v)\leftarrow V$ \\
   \Return $V, T$
   \EndFunction
\end{algorithmic}
\end{algorithm}
\vfill
\end{minipage}
\hfill
\begin{minipage}{.48\textwidth}
\centering
\begin{algorithm}[H]
\small
\caption{Backward Propagation Function}
\label{alg:backward}
\begin{algorithmic}
   \Function{Backward}{$v$, $M$, $T$, $E$, $t$}
   \State $E(v)\leftarrow t$
   \If {$v$.type = \emph{intersection} or \emph{union}}
   \For {each node $u$ in $v$.child}
   \State $E\leftarrow$\Call{Backward}{$u$, $M$, $T$, $E$, $t$}
   \EndFor
   \EndIf
   \If {$v$.type = \emph{relational}}
   \If {$v$.child is not a constant node}
   \State $u\leftarrow v$.child
   \State $t'\leftarrow \argmax\{T(u)^T * M[v.r][:][t]\}$
   \State $E\leftarrow$\Call{Backward}{$u$, $M$, $T$, $E$, $t'$}
   \EndIf
   \EndIf
   \If {$v$.type = \emph{anti-relational}}
   \If {$v$.child is not a constant node}
   \State $u\leftarrow v$.child
   \State $t'\leftarrow \argmax\{T(u)^T * (1^{|\mathcal{V}|\times1}-M[v.r][:][t])\}$
   \State $E\leftarrow$\Call{Backward}{$u$, $M$, $T$, $E$, $t'$}
   \EndIf
   \EndIf
   \Return $E$
   \EndFunction
\end{algorithmic}
\end{algorithm}
\begin{algorithm}[H]
\small
\caption{Query Computation Tree Optimization}
\label{alg:qto}
\begin{algorithmic}
   \State \textbf{Input:} neural adjacency matrix $M$, nodes in the query computation tree (each node has \emph{child} and \emph{type} members), and root $v_?$
   \State \textbf{Output: } optimal truth values $T$ on root (our defined $\mathbf{T}^*(v_?)$ vector), and the optimal assignments $E$ for $v_?=t$
   \State $T, E\leftarrow \{\}$
   \State $V, T\leftarrow$\Call{Forward}{$v_?$, $M$, $T$}
   \State $E\leftarrow$\Call{Backward}{$v_?$, $M$, $T$, $E$, $t$} \\
   \Return $T, E$
\end{algorithmic}
\end{algorithm}
\vfill
\end{minipage}
\end{minipage}

Alg.~\ref{alg:forward} shows the forward propagation function, which recursively calls itself on the child nodes to obtain the optimal truth values of the child subtrees, until reaching a leaf node (constant entity).
Alg.~\ref{alg:backward} shows the backward propagation function, which also recursively calls itself on the child nodes to compute the optimal assignments for the child subtrees, until reaching a leaf node.
Our QTO algorithm is shown in Alg.~\ref{alg:qto}, which calls the forward-propagation procedures to solve the optimization problem.

\section{Proof of \cref{thm:main}: Optimality of QTO}
\label{app:thm}

\thm*
\begin{proof}
    First, we show the optimality of the forward procedure by proving that the derived $\mathbf{T}^*(v)$ is indeed the maximum truth value of the subquery rooted at $v$.
    We prove it by induction: suppose it holds for the subquery rooted at child nodes of $v_?$, then according to the derivation process in Eq.~\ref{eq:f1}-\ref{eq:f4c}, it still holds for $v_?$.
    The trivial case is when the child node of $v_?$ is a constant entity, which also holds true according to Eq.~\ref{eq:f3c}, \ref{eq:f4c}.
    Therefore, according to the definition of $\mathbf{T}^*(v_?)$, the maximum truth value for the query can be obtained by Eq.~\ref{eq:opt_root}.

    Next, we show that the backward procedure can find a set of assignments that obtains the optimal value.
    We prove it by induction on the root of the query subtree.
    Suppose our backward procedure can find the optimal set of assignments for the subquery rooted at every child node $v$ of $v_?$, when the child node is assigned an arbitrary $e\in\mathcal{V}$. In other words, the truth value of such a set of assignments is $T(v)=T^*(v=e)$.
    Then we consider the procedure on the subquery rooted at $v_?$ when $v_?$ is assigned an arbitrary $e_t\in\mathcal{V}$. For type \Roman{1} (intersection), by Eq.~\ref{eq:qt1}, \ref{eq:b12}, the truth value $T(v_?)$ under the returned assignments is:
    \begin{equation}
        T(v_?) = \top_{1\leq i\leq K}(T(v_?^i)) = \top_{1\leq i\leq K}(T^*(v_?^i=e_t))
    \end{equation}
    while
    \begin{equation}
        T^*(v_?=e_t) = \max\{\top_{1\leq i\leq K}(T(v_?^i=e_t))\} = \top_{1\leq i\leq K}(T^*(v_?^i=e_t))
    \end{equation}
    Hence $T(v_?)=T^*(v_?=e_t)$. Similarly, it also holds for type \Roman{2} (union):
    \begin{equation}
        T(v_?) = \bot_{1\leq i\leq K}(T(v_?^i)) = \bot_{1\leq i\leq K}(T^*(v_?^i=e_t)) = T^*(v_?=e_t)
    \end{equation}
    For type \Roman{3} (relational projection), by Eq.~\ref{eq:qt2}, \ref{eq:b3}, and the induction hypothesis, we have
    \begin{equation}
    \begin{aligned}
        &T(v_?) = T^*(v_k=e)\ \top\ r(e, e_t),\ e=\argmax_{e\in\mathcal{V}}\{T^{*}(v_k=e)\ \top\ r(e, e_t)\} \\
        &\Rightarrow T(v_?) = \max_{e\in\mathcal{V}}\{T^*(v_k=e)\ \top\ r(e, e_t)\}
    \end{aligned}
    \end{equation}
    while by definition of $T^*$, it holds that
    \begin{equation}
        T^*(v_?=e_t) = \max\{T(v_k)\ \top\ r(v_k, e_t)\}
        = \max_{e\in\mathcal{V}}\{T^*(v_k=e)\ \top\ r(e, e_t)\}
    \end{equation}
    Thus $T(v_?)=T^*(v_?=e_t)$. Similarly, for type \Roman{4} (anti-relational projection), by Eq.~\ref{eq:qt2}, \ref{eq:b4}, it still holds:
    \begin{equation}
    \begin{aligned}
        &T(v_?) = T^*(v_k=e)\ \top\ (1-r(e, e_t)),\ e=\argmax_{e\in\mathcal{V}}\{T^{*}(v_k=e)\ \top\ (1-r(e, e_t))\} \\
        &\Rightarrow T(v_?) = \max_{e\in\mathcal{V}}\{T^*(v_k=e)\ \top\ (1-r(e, e_t))\}=T^*(v_?=e_t)
    \end{aligned}
    \end{equation}
    Hence, the induction hypothesis holds for $v_?$.
    The trivial case is when $v_?=c$ is a constant entity, and its truth value is simply $T(c)=T^*(c)=1$.
    Therefore, by Eq.~\ref{eq:opt_root}, the backward procedure from root $v_?$ returns a set of assignments that obtains a truth value of 
    \begin{equation}
        T(v_?)=T^*(v_?=e),\ e=\argmax_{e\in\mathcal{V}}\{T^*(v_?=e)\}
        \Rightarrow T(v_?)=\max_{e\in\mathcal{V}}\{T^*(v_?=e)\}
    \end{equation}
    which is the optimal truth value.
\end{proof}

\section{Discussions on Neural Adjacency Matrix}
\label{app:discussion}
\xhdr{Neural Adjacency Matrix}
An intriguing question is, can our method learn the parameters in the neural adjacency matrix?
In fact, our method supports optimization on the parameters, since each step is differentiable.
However, we should not directly optimize the neural adjacency matrix, since it does not generalize to missing links.
More specifically, during the learning procedure, the entries corresponding to already existing edges are pushed to 1 while other entries to 0, since such a matrix would induce a correct answer.
A more promising direction is to learn the calibration function that transforms the KGE scores to probabilities between $[0, 1]$ for the neural adjacency matrix, also known as KGE calibration~\cite{safavi2020evaluating}.
In our work, the calibration function is fixed, as shown in Eq.~\ref{eq:norm}, \ref{eq:nam}, but more functionals can be defined to learn soft ``threshold'' or ``stretching'' to better calibrate KGE scores to complex query answering.

\xhdr{Storage}
We store the $|\mathcal{R}|\times|\mathcal{V}|\times|\mathcal{V}|$ neural adjacency matrix as a sparse tensor since most of its entries are 0 (filtered out by threshold $\epsilon$). For example, with an $\epsilon$ of $0.0002$ on FB15k-237, only 1\% of all entries are nonzero. We study the overlap between the neural adjacency matrix and the truth adjacency matrix. On triplets in the training graph, the faithful rounding (Eq.~\ref{eq:nam}) ensures the corresponding entries in the neural adjacency matrix are 1. On triplets in valid+test graph, 99.1\% of the corresponding entries have a nonzero value (after filtering with $\epsilon=0.0002$), and the mean value of all the entries is 0.19. Meanwhile, on false triplets (no such edges in the full graph), only 1.0\% of the corresponding entries have a nonzero value, and the mean value of all the entries is 1e-5. The gap between the corresponding entries in the neural adjacency matrix for existing edges and non-existing edges is sufficient to distinguish between true and false triplets.

\xhdr{Scalability}
For a given KG, its neural adjacency matrix $\mathbf{M}$ is pre-computed by a pretrained KGE model and then saved for query answering.
The pre-computing step on FB15k, FB15k-237, and NELL995 take 40mins, 7mins, and 7hrs, which are all carried out on one RTX-3090 GPU.
We see that the pre-computing time scales to the number of entities and relations in the KG.
A possible way to mitigate this problem may be to first predict what relations each entity possesses based on the concept of the entity or other learnable property of the entity.
Another way is to simplify $\mathbf{M}$ as a block matrix, and compute each block in the corresponding subgraph.
We leave the scalability of QTO to large KGs for future work.

\section{Experiment Details}

\subsection{Dataset Statistics}
\label{app:stat}

\begin{table*}[!h]
\centering
\begin{tabular}{lccccc}
\toprule
\bf{Dataset} & \bf{\#Entities} & \bf{\#Relations} & \bf{\#Training edges} & \bf{\#Valid edges} & \bf{\#Test edges} \\
\midrule
FB15k & 14,951 & 1,345 & 483,142 & 50,000 & 59,071 \\
FB15k-237 & 14,505 & 237 & 272,115 & 17,526 & 20,438 \\
NELL995 & 63,361 & 200 & 114,213 & 14,324 & 14,267 \\
\toprule
\end{tabular}
\caption{Statistics of the three knowledge graph datasets.}
\label{tb:stat}
\end{table*}

Table~\ref{tb:stat} summarizes the statistics of the three datasets in our experiments.
Note that the inverse of each relation is also added to the graph, and can appear in the queries.

\subsection{Query Structures}
\label{app:query}

\begin{figure}[htbp]
\centering
\includegraphics[width=1\linewidth,trim=0 0 0 20,clip]{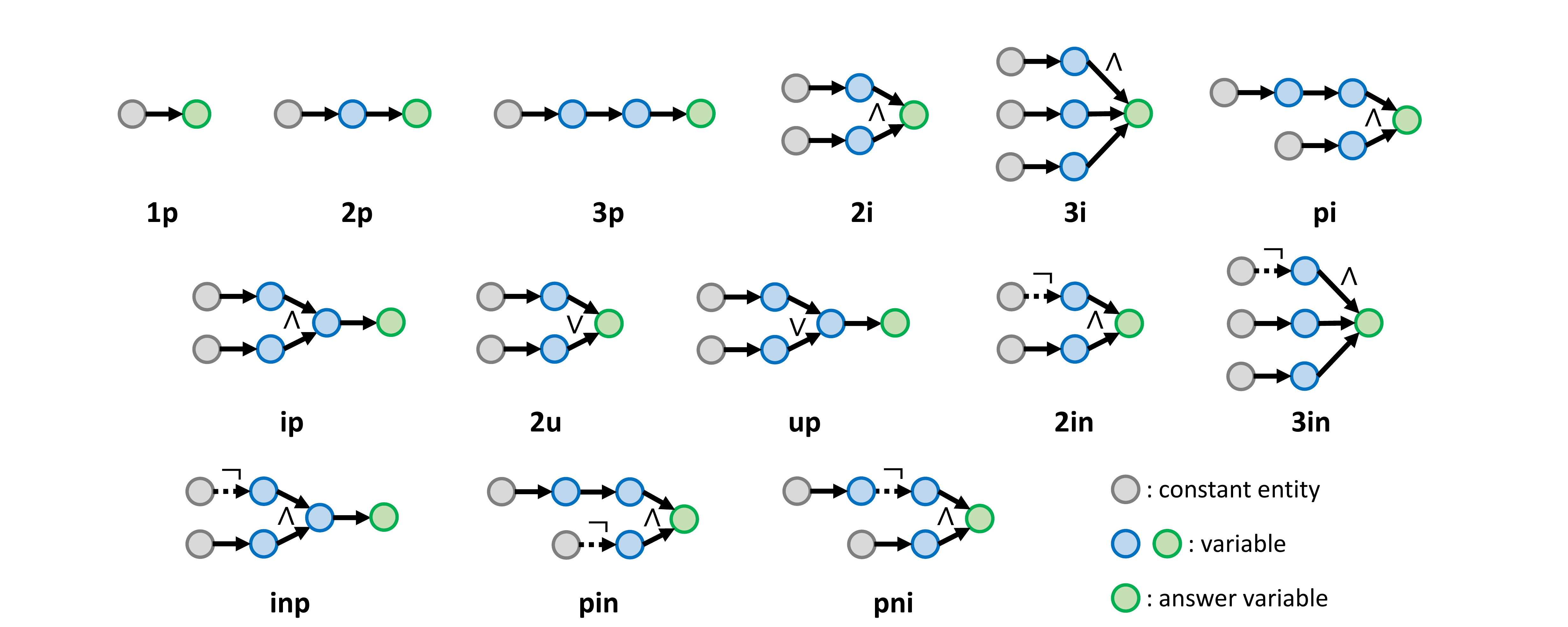}
\caption{Query structures, illustrated in their query computation tree representations.}
\label{fig:qtype}
\end{figure}

We use the 14 types of complex queries generated in~\cite{ren2019query2box, ren2020beta}. We show the query computation tree for each type of query in Figure~\ref{fig:qtype}.

\subsection{Implementation Details}
\label{app:detail}

\begin{table*}[!h]
\centering
\begin{tabular}{lcccccccc}
\toprule
 & \multicolumn{6}{c}{KGE (ComplEx)} & \multicolumn{2}{c}{QTO} \\
\cmidrule(l){2-7} \cmidrule(l){8-9}
& $d$ & $lr$ & $b$ & $\lambda$ & $w$ & epoch & $\epsilon$ & $\alpha$ \\
\midrule
FB15k & 1000 & 0.1 & 100 & 0.01 & 0.1 & 100 & 0.001 & 6 \\
FB15k-237 & 1000 & 0.1 & 1000 & 0.05 & 4 & 100 & 0.0002 & 3 \\
NELL995 & 1000 & 0.1 & 1000 & 0.05 & 0 & 100 & 0.0002 & 6 \\
\toprule
\end{tabular}
\caption{Hyperparameters of pretrained KGE and QTO.}
\label{tb:hyper}
\end{table*}

We provide the best hyperparameters of the pretrained KGE~\footnote{We utilize the KGE implementation from \href{https://github.com/facebookresearch/ssl-relation-prediction}{https://github.com/facebookresearch/ssl-relation-prediction}~\cite{chen2021relation}.} and QTO in Table~\ref{tb:hyper}.
The hyperparameters for KGE, which is a ComplEx model~\cite{trouillon2016complex} trained with N3 regularizor~\cite{lacroix2018canonical} and auxiliary relation prediction task~\cite{chen2021relation}, include embedding dimension $d$, learning rate $lr$, batch size $b$, regularization strength $\lambda$, auxiliary relation prediction weight $w$, and the number of epochs.
We recall that the hyperparameters in our QTO method include the threshold $\epsilon$ and the negation scaling coefficient $\alpha$.

\section{More Experimental Results}

\subsection{Hits@1 on Complex Query Answering}
\label{app:h@1}

\begin{table*}[!h]
    \centering
    \begin{adjustbox}{width=\textwidth}
    \begin{tabular}{lccccccccccccccccc}
        \toprule
        \bf{Model} & \bf{avg$_p$} & \bf{avg$_{ood}$} & \bf{avg$_n$} & \bf{1p} & \bf{2p} & \bf{3p} & \bf{2i} & \bf{3i} & \bf{pi} & \bf{ip} & \bf{2u} & \bf{up} & \bf{2in} & \bf{3in} & \bf{inp} & \bf{pin} & \bf{pni} \\
        \midrule
        \multicolumn{18}{c}{FB15k} \\
        \midrule
        GQE & 16.6 & 11.0 & - & 34.2 & 8.3 & 5.0 & 23.8 & 34.9 & 15.5 & 11.2 & 11.5 & 5.6 & - & - & - & - & - \\
        Query2Box & 26.8 & 18.7 & - & 52.0 & 12.7 & 7.8 & 40.5 & 53.4 & 26.7 & 16.7 & 22.0 & 9.4 & - & - & - & - & - \\
        BetaE & 31.3 & 24.2 & 5.2 & 52.0 & 17.0 & 16.9 & 43.5 & 55.3 & 32.3 & 19.3 & 28.1 & 16.9 & 6.4 & 6.7 & 5.5 & 2.0 & 5.3 \\
        CQD-CO & 39.7 & 26.4 & - & 85.8 & 17.8 & 9.0 & 67.6 & 71.7 & 34.5 & 24.5 & 30.9 & 15.5 & - & - & - & - & - \\
        CQD-Beam & 51.9 & 42.7 & - & 85.8 & 48.6 & 22.5 & 67.6 & 71.7 & 51.7 & 62.3 & 31.7 & 25.0 & - & - & - & - & - \\
        ConE & 39.6 & 33.0 & 7.3 & 62.4 & 23.8 & 20.4 & 53.6 & 64.1 & 39.6 & 25.6 & 44.9 & 21.7 & 9.4 & 9.1 & 6.0 & 4.3 & 7.5 \\
        GNN-QE & 67.3 & 62.2 & 28.6 & 86.1 & \bf{63.5} & \bf{52.5} & 74.8 & \bf{80.1} & 63.6 & 65.1 & 67.1 & 53.0 & 35.4 & 33.1 & 33.8 & 18.6 & \bf{21.8} \\
        \midrule
        QTO & \bf{68.6} & \bf{66.0} & \bf{39.1} & \bf{86.6} & 60.9 & 51.9 & \bf{75.3} & 79.0 & \bf{69.2} & \bf{68.8} & \bf{71.8} & \bf{54.3} & \bf{50.3} & \bf{50.3} & \bf{39.1} & \bf{36.9} & 18.8 \\
        \midrule[0.08em]
        \multicolumn{18}{c}{FB15k-237} \\
        \midrule
        GQE & 8.8 & 4.9 & - & 22.4 & 2.8 & 2.1 & 11.7 & 20.9 & 8.4 & 5.7 & 3.3 & 2.1 & - & - & - & - & - \\
        Query2Box & 12.3 & 7.0 & - & 28.3 & 4.1 & 3.0 & 17.5 & 29.5 & 12.3 & 7.1 & 5.2 & 3.3 & - & - & - & - & - \\
        BetaE & 13.4 & 7.9 & 2.8 & 28.9 & 5.5 & 4.9 & 18.3 & 31.7 & 14.0 & 6.7 & 6.3 & 4.6 & 1.5 & 7.7 & 3.0 & 0.9 & 0.9 \\
        CQD-CO & 14.7 & 9.5 & - & 36.6 & 4.7 & 3.0 & 20.7 & 29.6 & 15.5 & 9.9 & 8.6 & 4.0 & - & - & - & - & - \\
        CQD-Beam & 15.1 & 9.7 & - & 36.6 & 6.3 & 4.3 & 20.7 & 29.6 & 13.5 & 12.1 & 8.7 & 4.3 & - & - & - & - & - \\
        ConE & 15.6 & 9.5 & 2.2 & 31.9 & 6.9 & 5.3 & 21.9 & 36.6 & 17.0 & 7.8 & 8.0 & 5.3 & 1.8 & 3.7 & 3.4 & 1.3 & 1.0 \\
        GNN-QE & 19.1 & 13.0 & 4.3 & 32.8 & 8.2 & 6.5 & 27.7 & 44.6 & 22.4 & 12.3 & 9.8 & 7.6 & 4.1 & 8.1 & 4.1 & 2.5 & \bf{2.7} \\
        \midrule
        QTO & \bf{25.4} & \bf{19.9} & \bf{8.3} & \bf{39.5} & \bf{14.3} & \bf{14.7} & \bf{33.2} & \bf{47.2} & \bf{29.0} & \bf{20.6} & \bf{15.1} & \bf{14.8} & \bf{8.6} & \bf{15.9} & \bf{8.5} & \bf{6.4} & 2.0 \\
        \midrule[0.08em]
        \multicolumn{18}{c}{NELL-995} \\
        \midrule
        GQE & 9.9 & 6.9 & - & 15.4 & 6.7 & 5.0 & 14.3 & 20.4 & 10.6 & 9.0 & 2.9 & 5.0 & - & - & - & - & - \\
        Query2Box & 14.1 & 8.8 & - & 23.8 & 8.7 & 6.9 & 20.3 & 31.5 & 14.3 & 10.7 & 5.0 & 6.0 & - & - & - & - & - \\
        BetaE & 17.8 & 9.6 & 2.1 & 43.5 & 8.1 & 7.0 & 27.2 & 36.5 & 17.4 & 9.3 & 6.9 & 4.7 & 1.6 & 2.2 & 4.8 & 0.7 & 1.2 \\
        CQD-CO & 21.3 & 13.9 & - & 51.2 & 11.8 & 9.0 & 28.4 & 36.3 & 22.4 & 15.5 & 9.9 & 7.6 & - & - & - & - & - \\
        CQD-Beam & 21.0 & 13.2 & - & 51.2 & 14.3 & 6.3 & 28.4 & 36.3 & 18.1 & 17.4 & 10.2 & 7.2 & - & - & - & - & - \\
        ConE & 19.8 & 11.6 & 2.2 & 43.6 & 10.7 & 9.0 & 28.6 & 39.8 & 19.2 & 11.4 & 9.0 & 6.6 & 1.4 & 2.6 & 5.2 & 0.8 & 1.2 \\
        GNN-QE & 21.5 & 13.2 & 3.6 & 43.5 & 12.9 & 9.9 & \bf{32.5} & \bf{42.4} & \bf{23.5} & 12.9 & 8.8 & 7.4 & 3.2 & 5.9 & 5.4 & 1.6 & 2.0 \\
        \midrule
        QTO & \bf{24.8} & \bf{16.7} & \bf{6.1} & \bf{51.6} & \bf{17.1} & \bf{15.3} & 32.1 & 40.8 & 23.1 & \bf{19.9} & \bf{12.3} & \bf{11.3} & \bf{5.6} & \bf{9.4} & \bf{9.9} & \bf{3.5} & \bf{2.1} \\
        \bottomrule
    \end{tabular}
    \end{adjustbox}
    \caption{Test Hits@1 results (\%) on complex query answering across all query types. avg$_p$ is the average on EPFO queries; avg$_{ood}$ is the average on out-of-distribution (OOD) queries; avg$_n$ is the average on queries with negation.}
    \label{tb:h@1}
\end{table*}

Table~\ref{tb:h@1} reports the Hits@1 result on complex query answering.
On Hits@1 metric, QTO outperforms previous Sota method GNN-QE by an average of 8.5\%, 39.8\%, and 20\% over all query types on the three datasets.

\subsection{More Plots on Complex Query Answering w.r.t. 1-hop Query Answering}
\label{app:plot}

\begin{figure}[htbp]
    \centering
    \subfigure[2p-1p]{
        \includegraphics[width=2.1in]{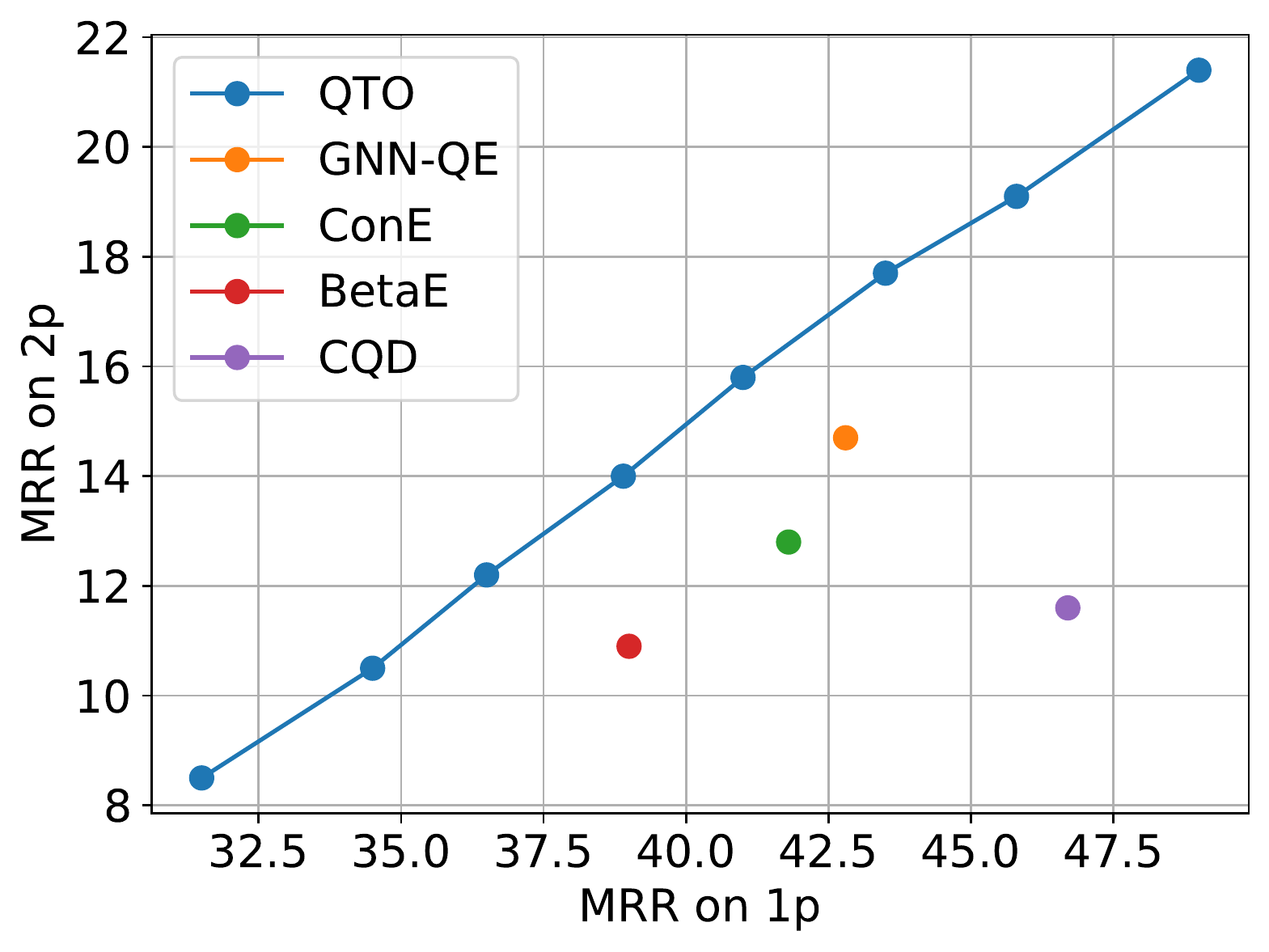}
    }
    \subfigure[3p-1p]{
	\includegraphics[width=2.1in]{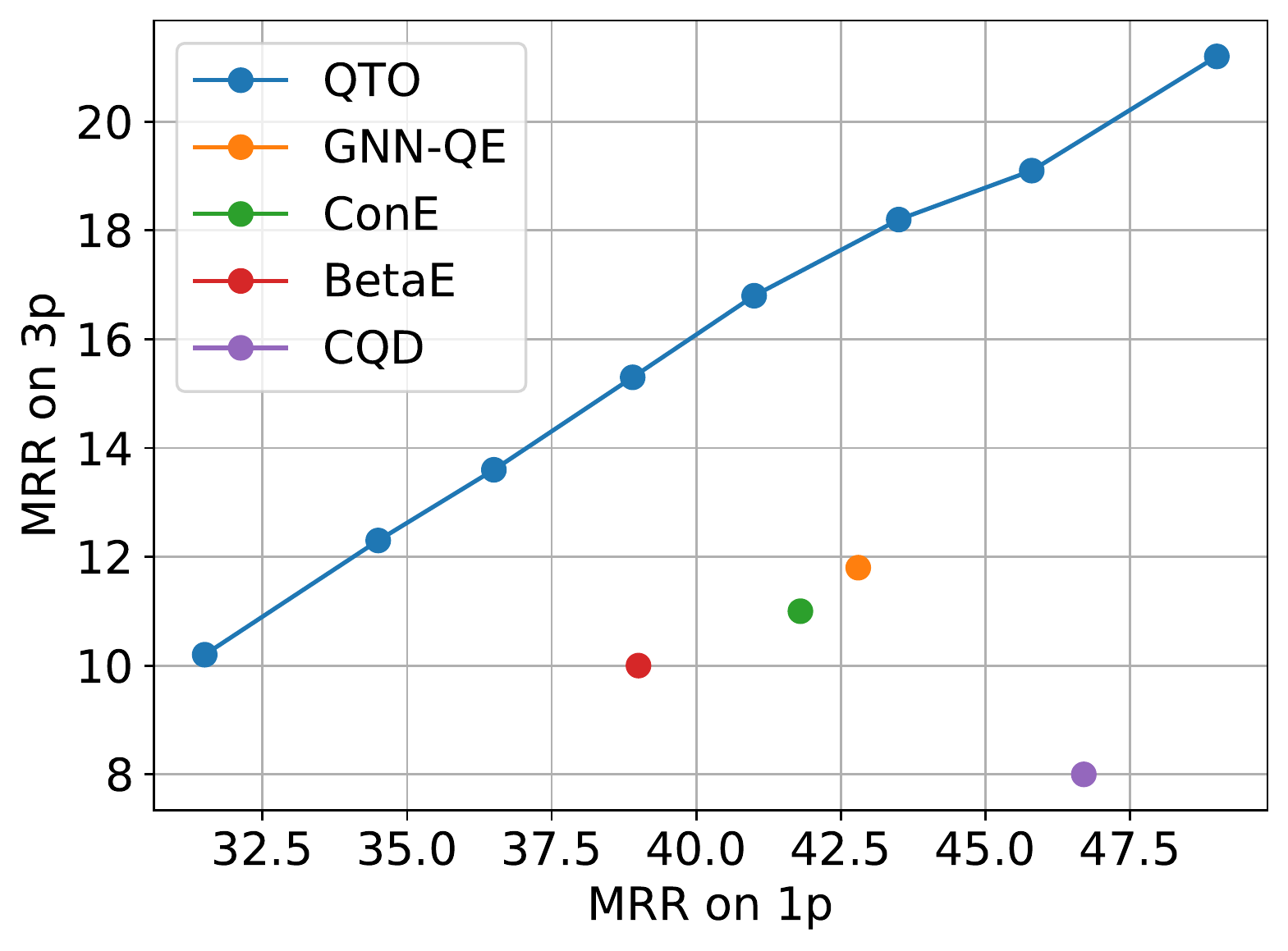}
    }
    \subfigure[2i-1p]{
    	\includegraphics[width=2.1in]{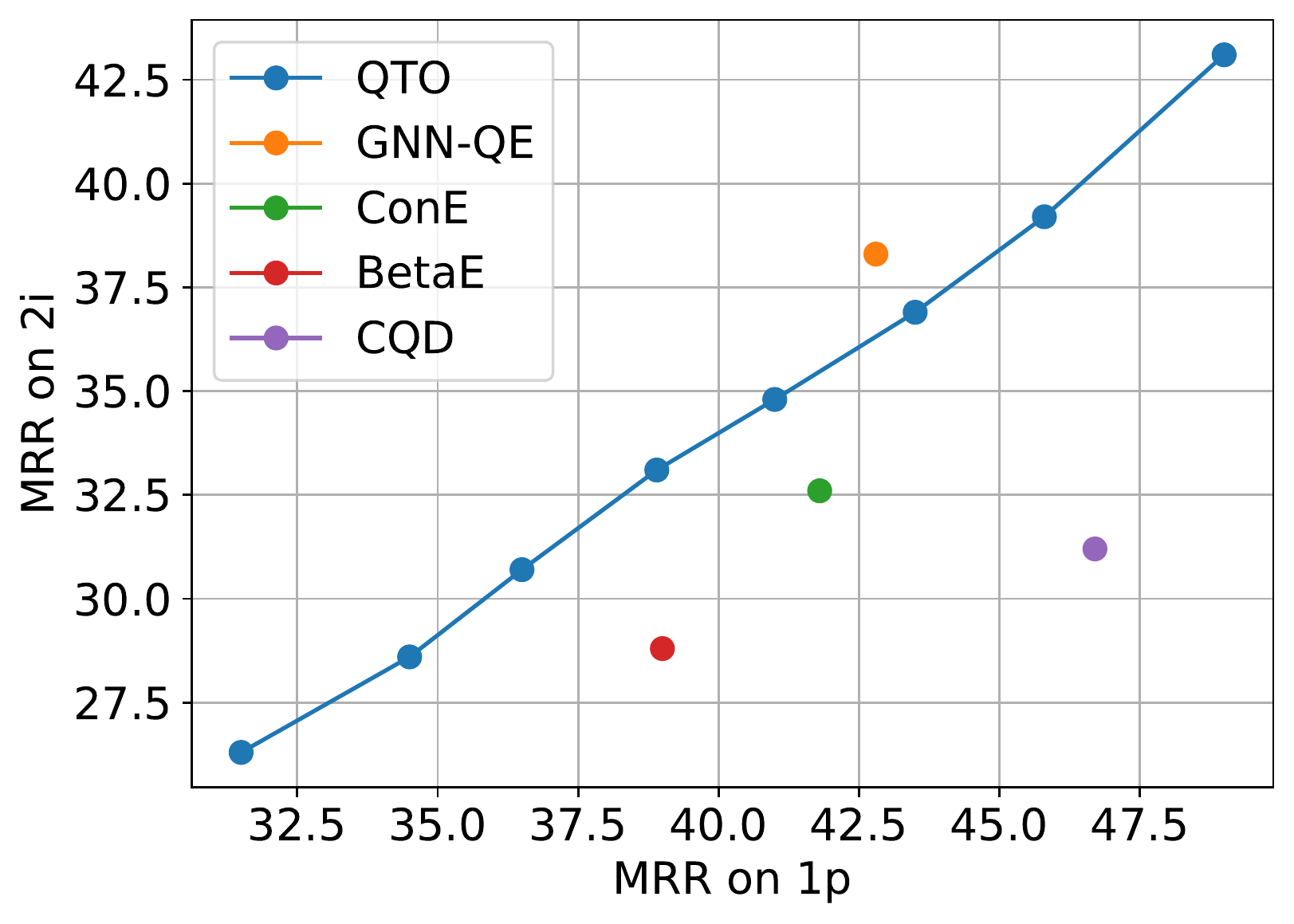}
    }
    \quad
    \subfigure[3i-1p]{
    	\includegraphics[width=2.1in]{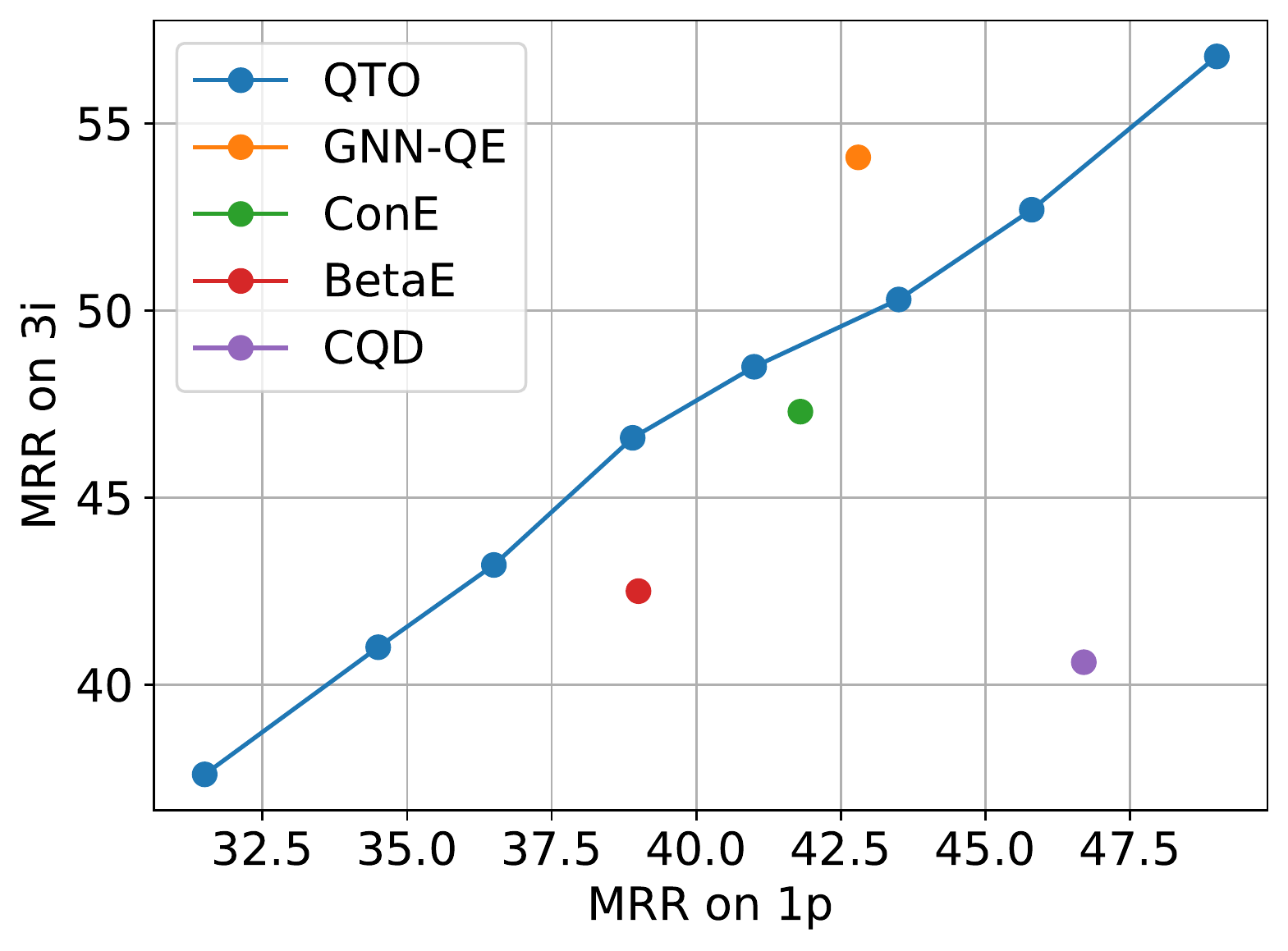}
    }
    \subfigure[pi-1p]{
	\includegraphics[width=2.1in]{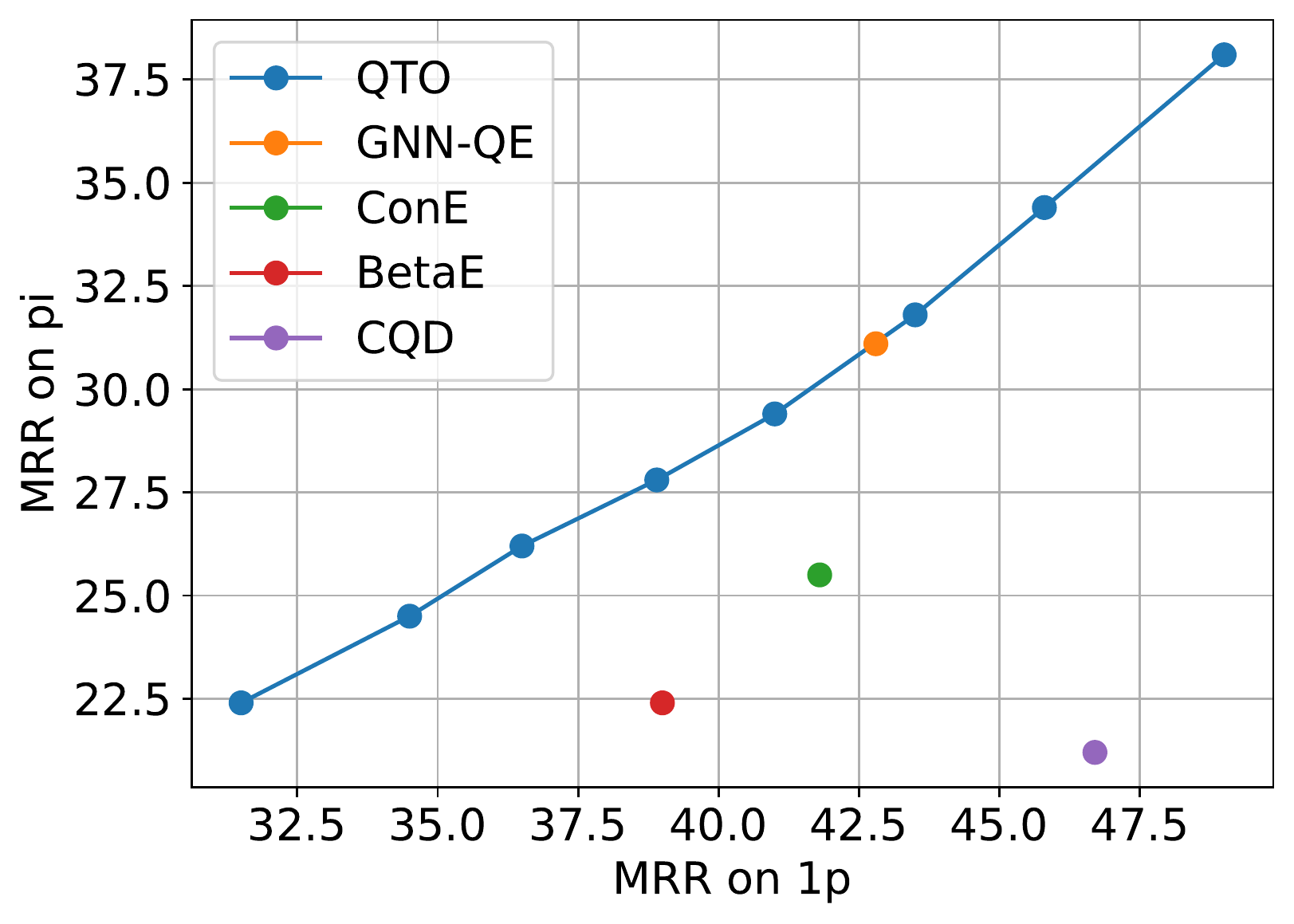}
    }
    \subfigure[ip-1p]{
	\includegraphics[width=2.1in]{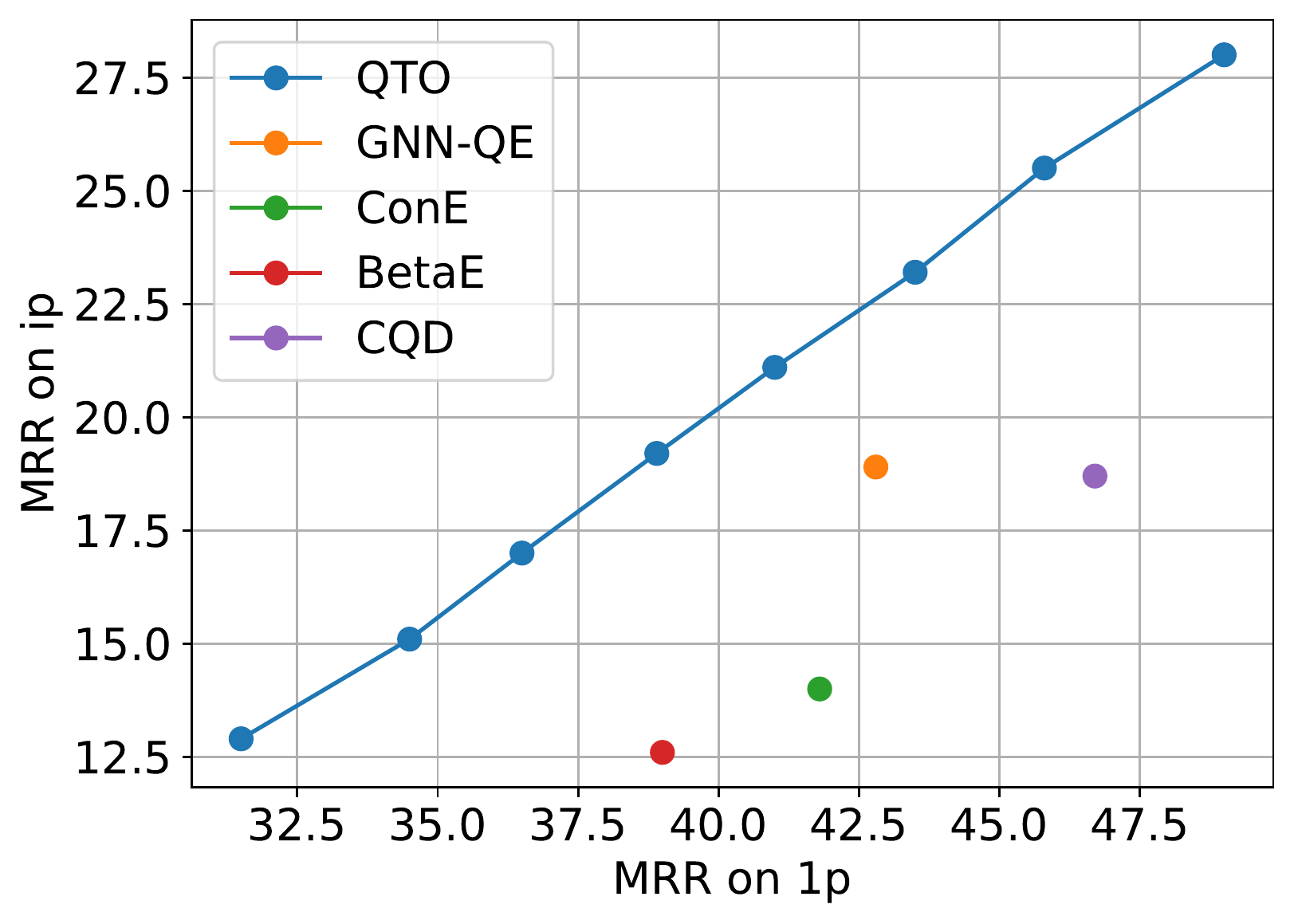}
    }
    \quad
    \subfigure[2u-1p]{
	\includegraphics[width=2.1in]{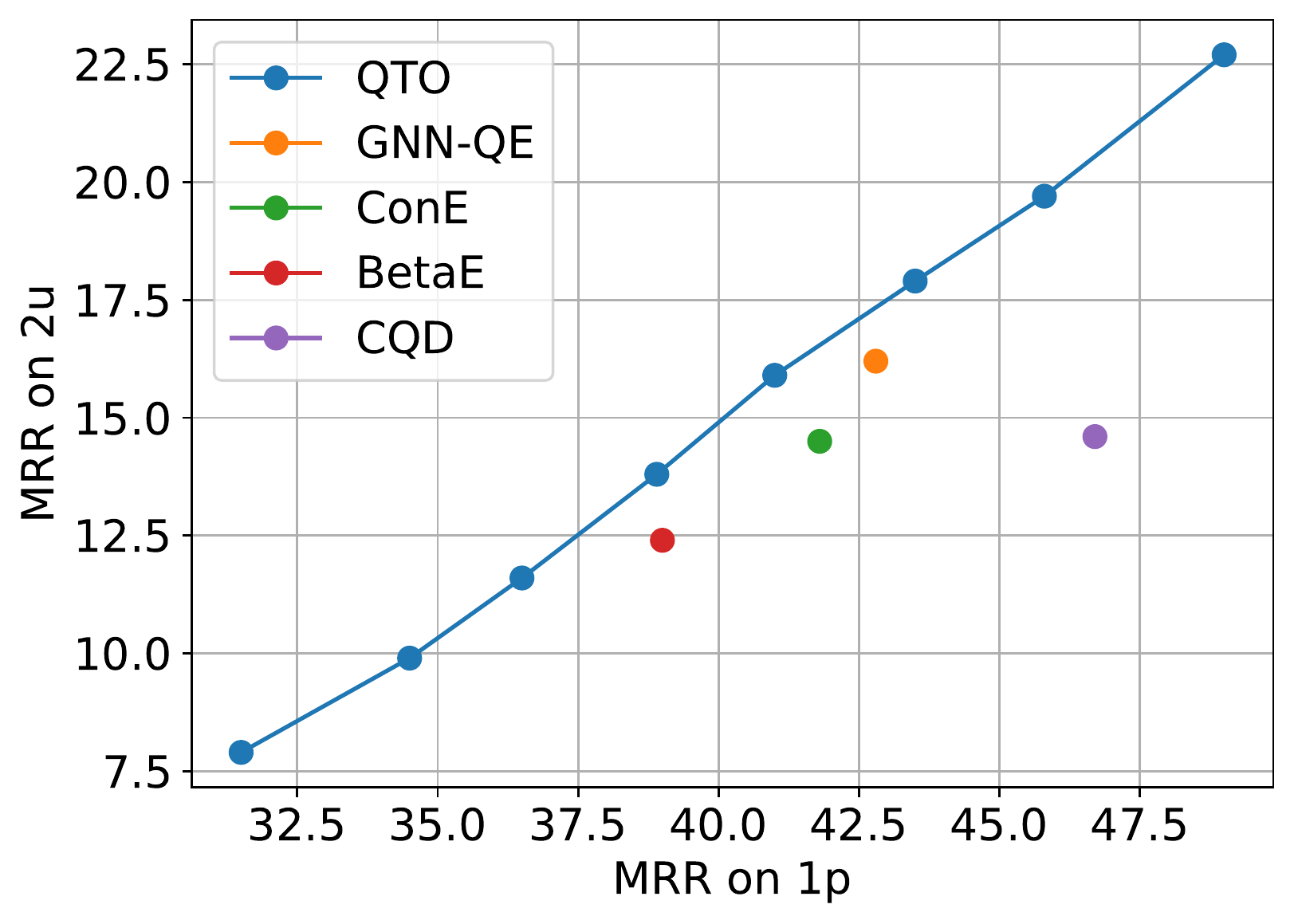}
    }
    \subfigure[up-1p]{
	\includegraphics[width=2.1in]{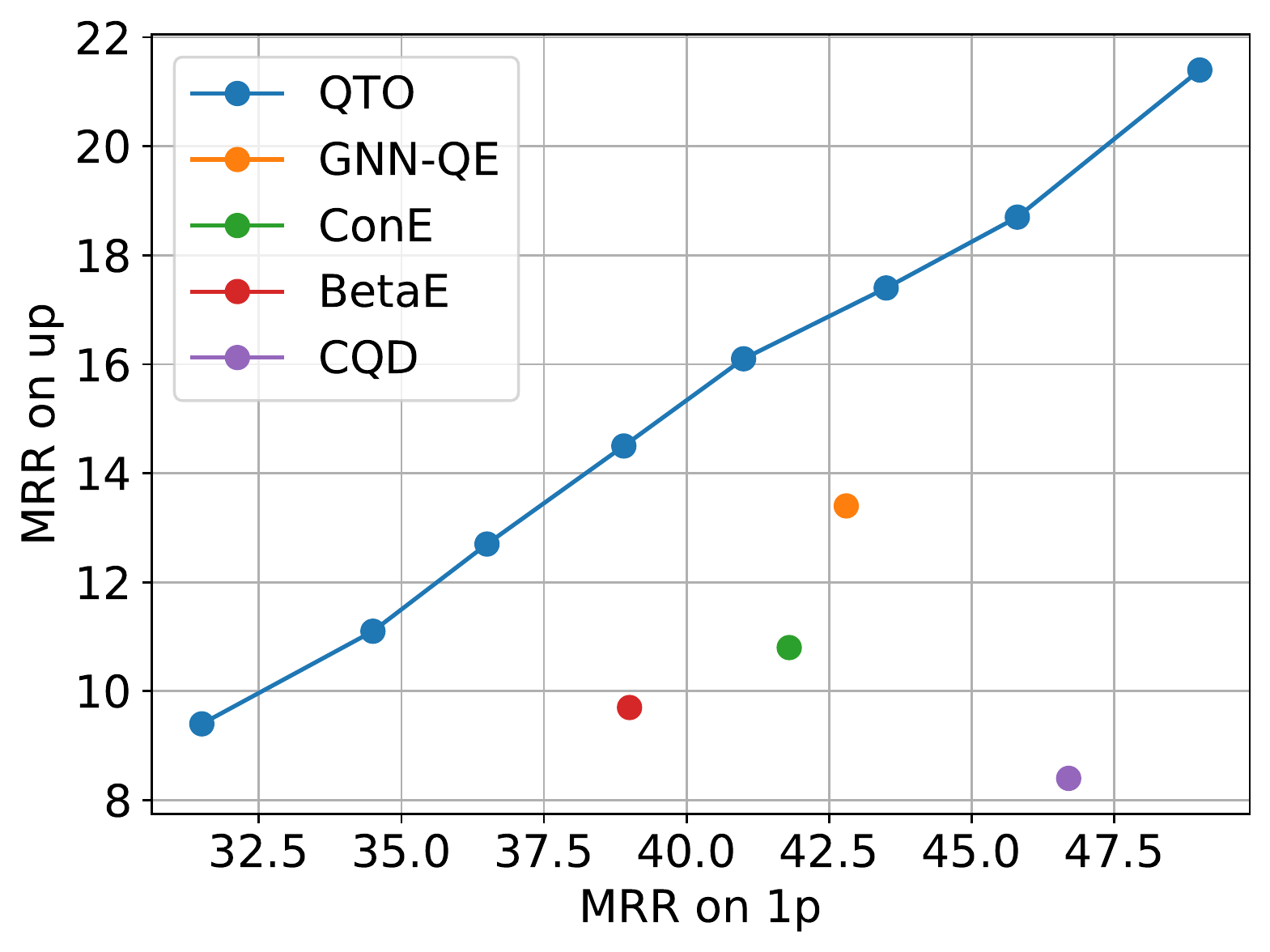}
    }
    \subfigure[2in-1p]{
	\includegraphics[width=2.1in]{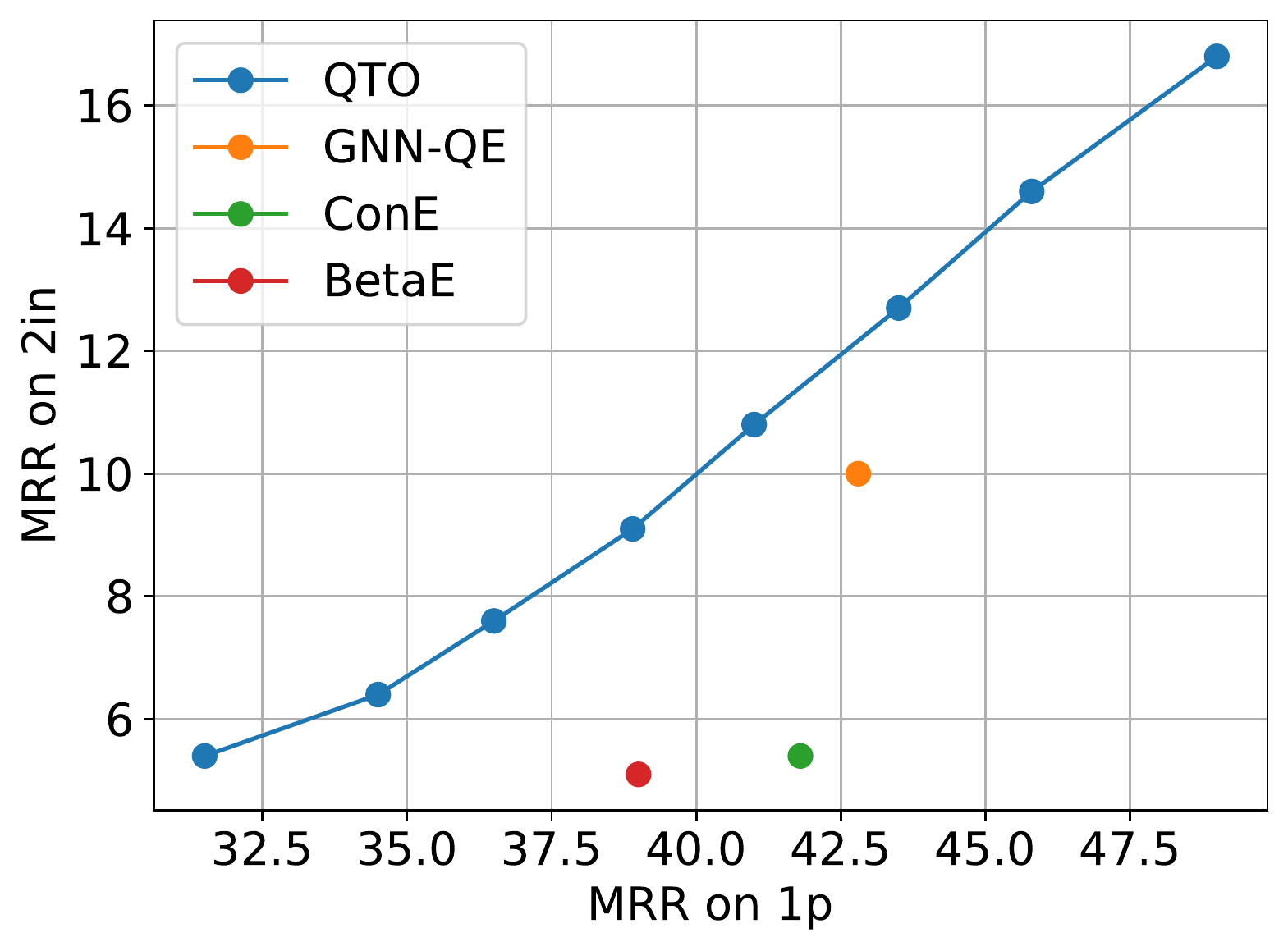}
    }
    \quad
    \subfigure[3in-1p]{
	\includegraphics[width=2.1in]{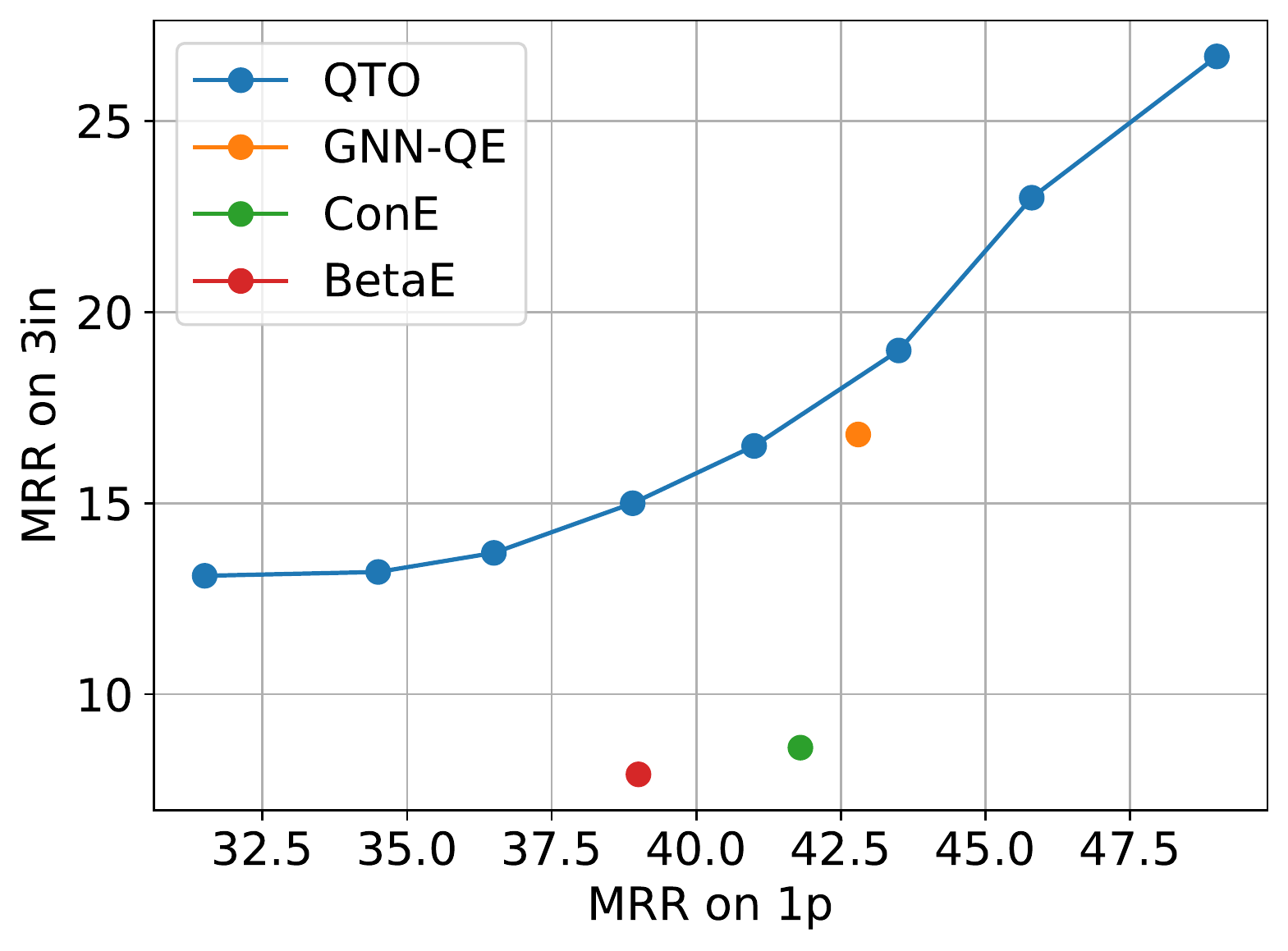}
    }
    \subfigure[inp-1p]{
	\includegraphics[width=2.1in]{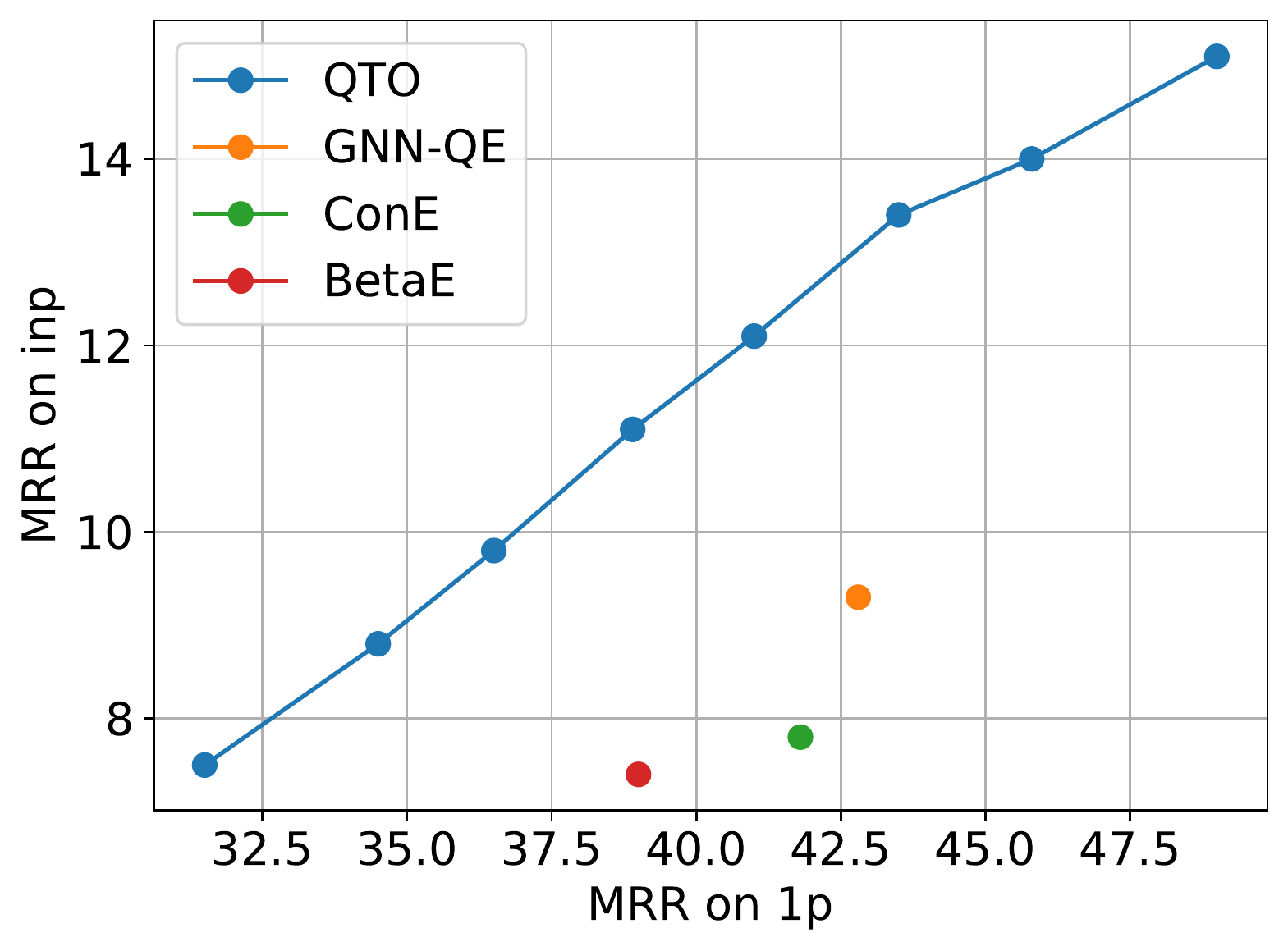}
    }
    \subfigure[pin-1p]{
	\includegraphics[width=2.1in]{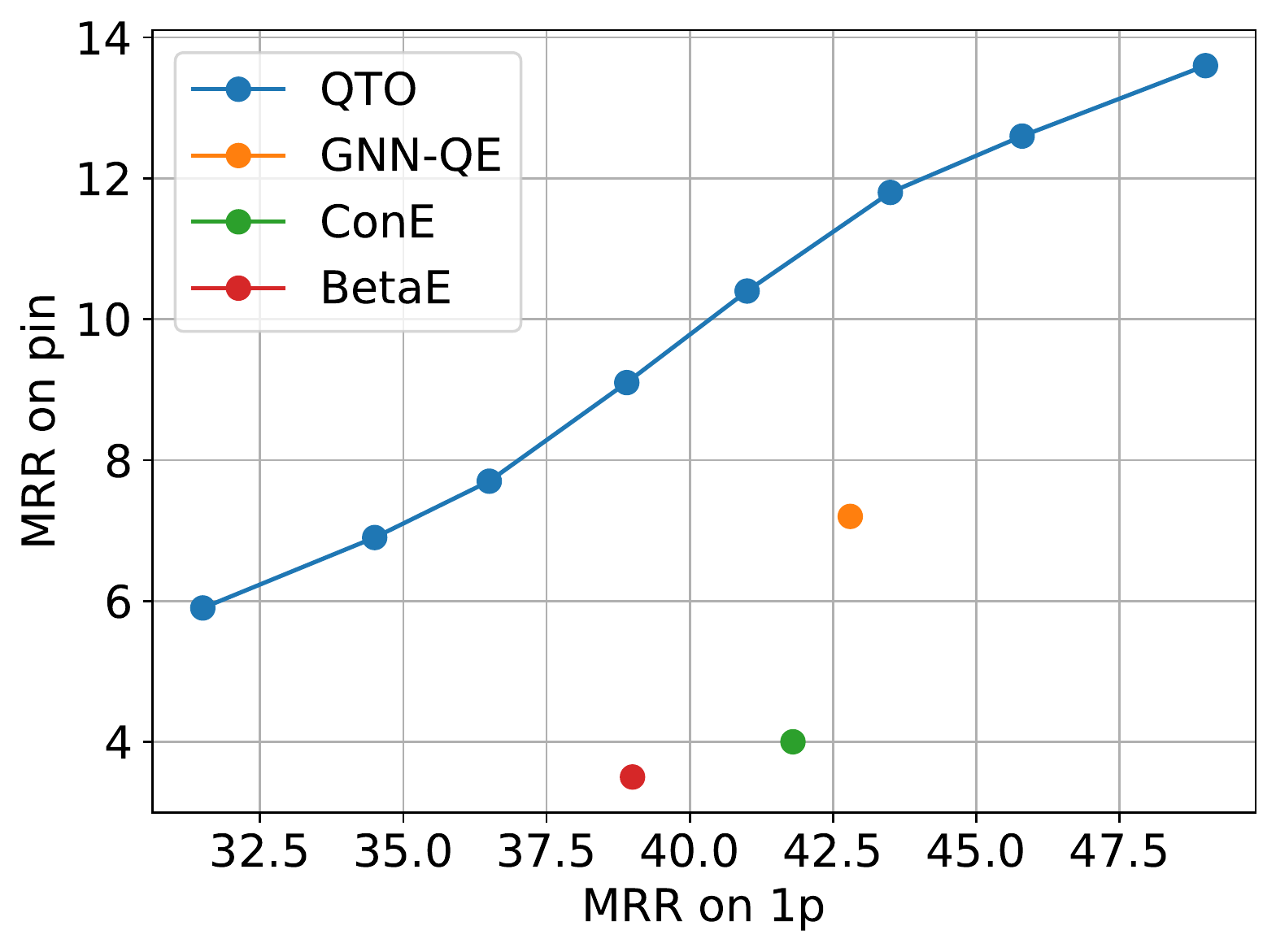}
    }
    \caption{Test MRR on each type of queries w.r.t. MRR on 1p (one-hop) queries, evaluated on FB15k-237.}
    \label{fig:all-1p}
\end{figure}

Figure~\ref{fig:all-1p} reports the Test MRR on each type of complex queries w.r.t. the MRR performance on one-hop queries. We observe a similar trend across all types of queries, verifying our conclusion in Sec.~\ref{sec:result}.

\subsection{Hyperparameter Analysis}
\label{app:hyper}
Table~\ref{tb:eps}, \ref{tb:alpha} shows the effect of $\epsilon$ and $\alpha$ on FB15k-237.

\begin{table}[htbp]
    \centering
    \begin{minipage}{0.4\linewidth}
        \centering
        \begin{tabular}{lcccc}
            \toprule
            $\epsilon$ & $0.1$ & $0.01$ & $0.001$  & $0.0002$ \\
            \midrule
            memory usage & 19M & 82M & 817M & 20G \\
            {\bf avg$_p$} (\%) & 24.5 & 31.5 & 33.1 & 33.5 \\
            \toprule
        \end{tabular}
        \caption{Effect of $\epsilon$ on memory usage and avg$_p$ MRR.}
        \label{tb:eps}
    \end{minipage}
    \qquad
    \begin{minipage}{0.5\linewidth}
        \centering
        \begin{tabular}{lcccccc}
            \toprule
            $\alpha$ & $0.5$ & $1$ & $3$ & $5$ & $7$ & $9$ \\
            \midrule
            {\bf avg$_n$} (\%) & 7.4 & 13.8 & 15.5 & 14.9 & 13.8 & 13.2 \\
            \toprule
        \end{tabular}
        \caption{Effect of $\alpha$ on avg$_n$ MRR.}
        \label{tb:alpha}
    \end{minipage}
\end{table}

\subsection{QTO with Different One-hop Link Predictors}
\label{app:kge}
Table~\ref{tb:kge} reports the performance of QTO with different KGE models on FB15k-237.

\begin{table}[t]
    \centering
    \begin{tabular}{lcccccccccc}
        \toprule
        & {\bf avg$_p$} & {\bf 1p} & {\bf 2p} & {\bf 3p} & {\bf 2i} & {\bf 3i} & {\bf pi} & {\bf ip} & {\bf 2u} & {\bf up} \\
        \midrule
        \multicolumn{11}{c}{QTO's performance under different KGE models} \\
        \midrule
        QTO+TransE & 22.1 & 43.2 & 11.5 & 9.7 & 28.1 & 41.5 & 23.3 & 17.8 & 13.4 & 9.7 \\
        QTO+RotatE & 22.4 & 43.8 & 11.7 & 9.8 & 28.7 & 42.2 & 23.6 & 18.0 & 14.1 & 9.7 \\
        QTO+ComplEx (w/ N3\&RP) & 33.5 & 49.0 & 21.4 & 21.2 & 43.1 & 56.8 & 38.1 & 28.0 & 22.7 & 21.4 \\
        \midrule
        \multicolumn{11}{c}{Comparison with prior methods under the same KGE model} \\
        \midrule
        CQD-Beam & 22.3 & 46.7 & 11.6 & 8.0 & 31.2 & 40.6 & 21.2 & 18.7 & 14.6 & 8.4 \\
        QTO+ComplEx (w/ N3) & 28.7 & 46.7 & 16.0 & 15.1 & 38.7 & 51.2 & 32.8 & 23.6 & 18.1 & 15.6 \\
        GNN-QE & 26.8 & 42.8 & 14.7 & 11.8 & 38.3 & 54.1 & 31.1 & 18.9 & 16.2 & 13.4 \\
        QTO+NBFNet & 27.3 & 51.7 & 17.5 & 13.3 & 31.8 & 43.0 & 28.9 & 24.2 & 22.3 & 12.6 \\
        \bottomrule
    \end{tabular}
    \caption{MRR (\%) on FB15k-237, RP is short for the relation prediction task~\cite{chen2021relation} during KGE training.}
    \label{tb:kge}
\end{table}

\subsection{More Results on Interpretation Study}
\label{app:inter}

\begin{table}[t]
    \centering
    \begin{tabular}{lcccccccc}
        \toprule
        & \bf{2p} & \bf{3p} & \bf{pi} & \bf{ip} & \bf{up} & \bf{inp} & \bf{pin} & \bf{pni} \\
        \midrule
        \multicolumn{9}{c}{FB15k} \\
        \midrule
        on Hits@1 & 98.2 & 96.3 & 98.2 & 98.2 & 98.2 & 96.2 & 98.6 & 97.8 \\
        on Hits@3 & 97.4 & 94.8 & 97.3 & 97.1 & 97.3 & 94.6 & 97.7 & 98.3 \\
        on Hits@10 & 96.0 & 92.7 & 96.6 & 96.0 & 95.7 & 92.8 & 96.3 & 98.7 \\
        on All & 90.5 & 85.3 & 94.1 & 91.7 & 89.2 & 81.8 & 90.5 & 97.7 \\
        \midrule
        \multicolumn{9}{c}{NELL995} \\
        \midrule
        on Hits@1 & 90.7 & 81.8 & 92.3 & 92.2 & 91.1 & 74.9 & 89.7 & 91.5 \\
        on Hits@3 & 88.3 & 77.2 & 91.6 & 89.9 & 88.8 & 71.5 & 86.4 & 93.8 \\
        on Hits@10 & 86.1 & 75.2 & 90.8 & 88.3 & 85.6 & 68.6 & 84.5 & 95.2 \\
        on All & 76.4 & 64.4 & 81.5 & 82.4 & 77.2 & 54.5 & 74.3 & 95.1 \\
        \bottomrule
    \end{tabular}
    \caption{Accuracy (\%) of intermediate variable interpretation given the Hits@K answers that QTO predicts and all true answers, evaluated on FB15k and NELL995.}
    \label{tb:interpretability1}
\end{table}

Table~\ref{tb:interpretability1} reports the accuracy on intermediate variable interpretation on FB15k and NELL995.
We observe that QTO can provide valid explanations for over 90\% of the Hits@1 answers it predicts.

\subsection{Case Study on QTO's Interpretation}
\label{app:case}

\begin{figure}[htbp]
    \centering
    \subfigure[Case study on 2p query]{
        \includegraphics[width=1\linewidth,trim=50 120 50 120,clip]{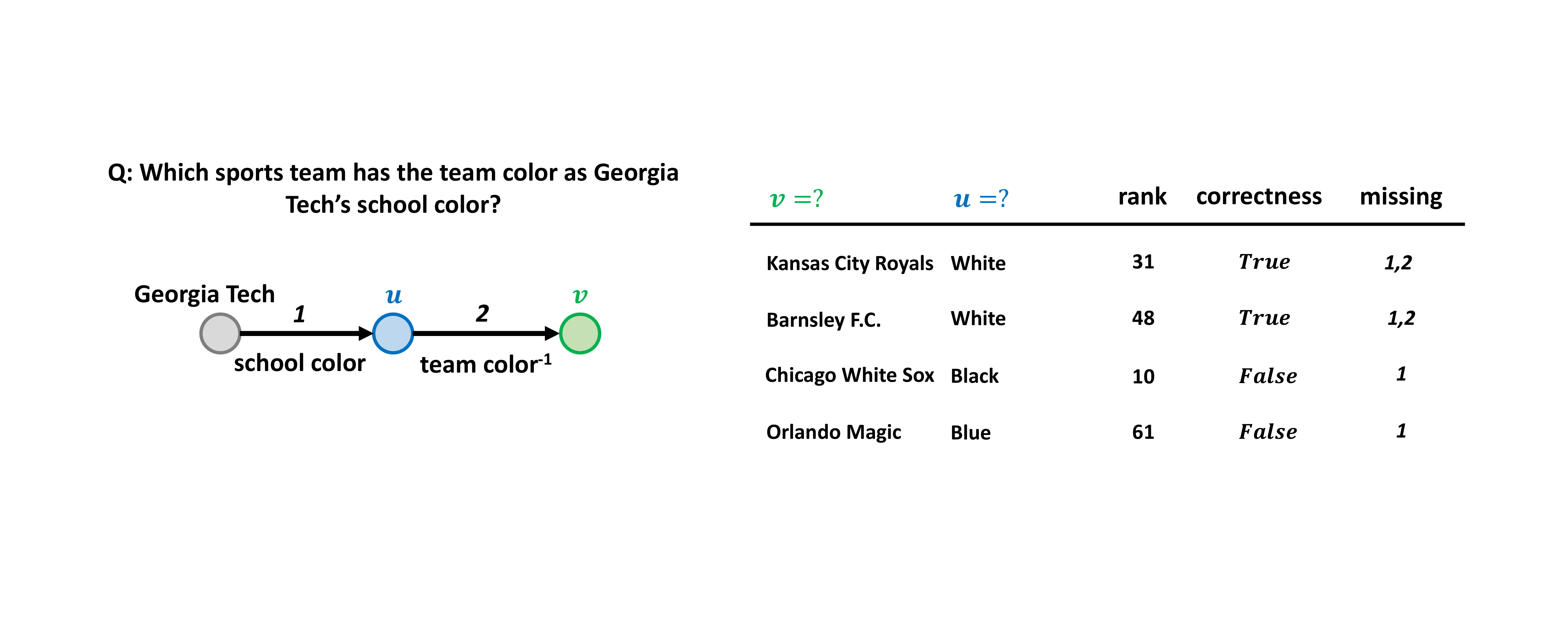}
    }
    \quad
    \subfigure[Case study on pi query]{
        \includegraphics[width=1\linewidth,trim=50 120 50 120,clip]{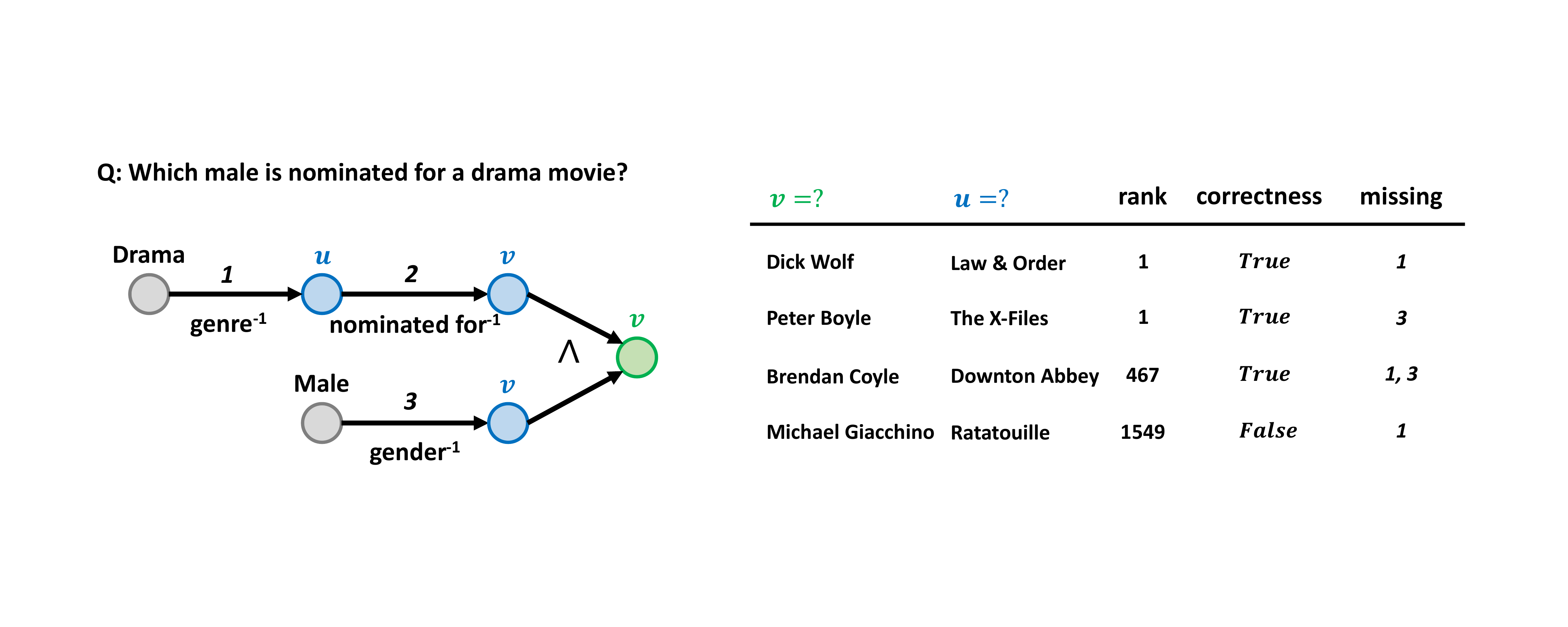}
    }
    \quad
    \subfigure[Case study on ip query]{
        \includegraphics[width=1\linewidth,trim=50 120 50 120,clip]{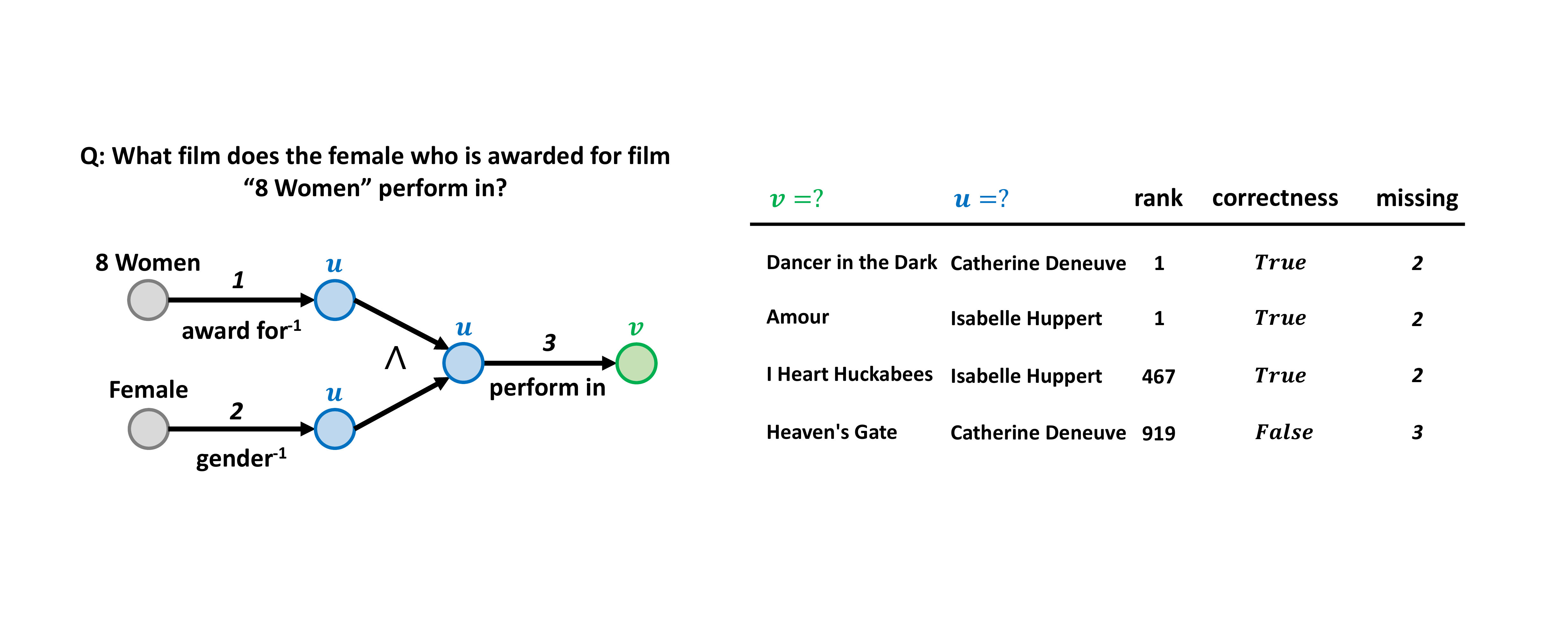}
    }
    \quad
    \subfigure[Case study on pni query]{
        \includegraphics[width=1\linewidth,trim=50 120 50 120,clip]{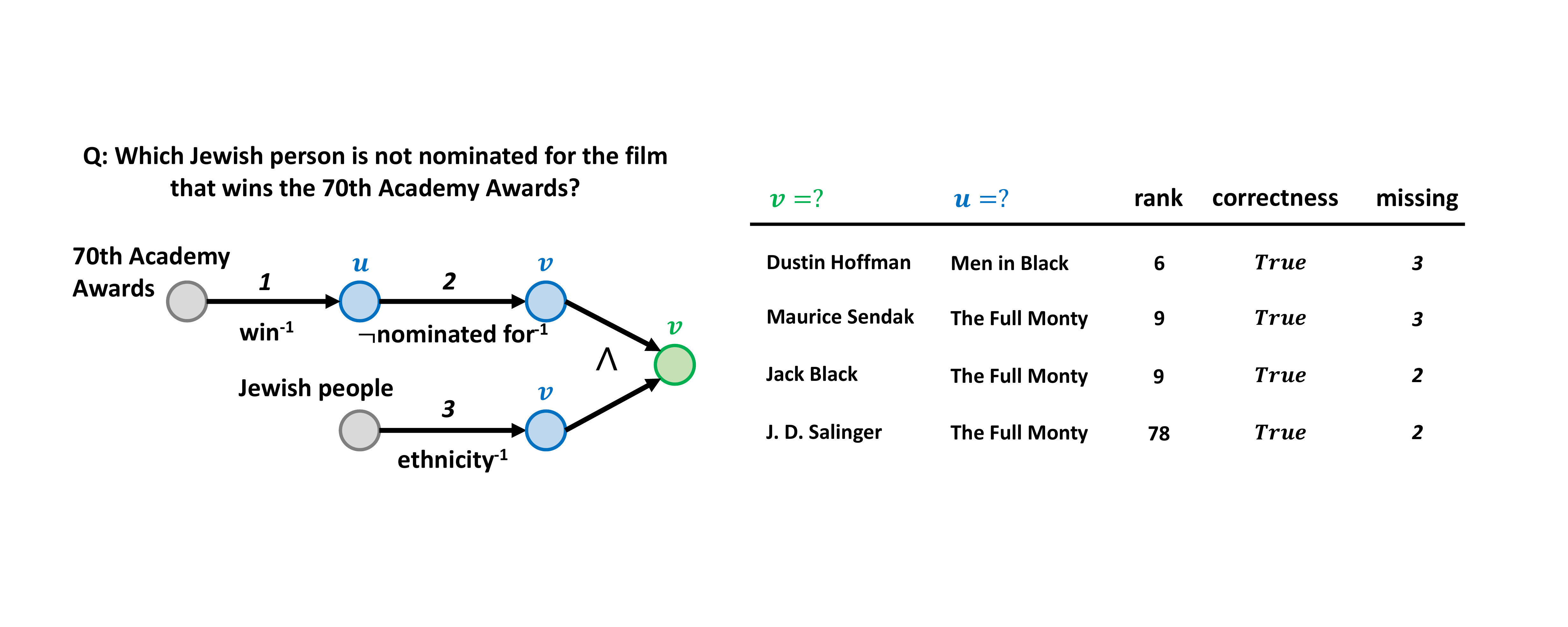}
    }
\caption{Case study of QTO's interpretation.}
\label{fig:case}
\end{figure}

We report case studies on QTO's interpretation in Figure~\ref{fig:case}. We include both success cases, where the assignments can be verified by the full graph, and failure cases where the assignments are not viable.
For each set of assignments, we provide the rank of its answer in QTO's prediction, the correctness of the assignments, and the missing link in the training graph that need to be predicted for deducing the answer.

\subsection{More Results on Cardinality Prediction}
\label{app:card}

\begin{table}[!h]
    \centering
    \begin{tabular}{lcccccccccc}
        \toprule
        \bf{Method} & \bf{avg} & \bf{1p} & \bf{2p} & \bf{3p} & \bf{2i} & \bf{3i} & \bf{pi} & \bf{ip} & \bf{2u} & \bf{up} \\
        \midrule
        \multicolumn{11}{c}{FB15k} \\
        \midrule
        GNN-QE & 37.1 & 34.4 & 29.7 & 34.7 & 39.1 & 57.3 & 47.8 & 34.6 & 13.5 & 26.5 \\
        QTO & \bf{23.1} & \bf{11.7} & \bf{24.1} & \bf{29.5} & \bf{25.5} & \bf{27.2} & \bf{29.3} & \bf{25.2} & \bf{12.4} & \bf{22.7} \\
        \midrule
        \multicolumn{11}{c}{NELL995} \\
        \midrule
        GNN-QE & 44.0 & 61.9 & 38.2 & 47.1 & 56.6 & 72.3 & 49.5 & 45.8 & 19.9 & 36.2 \\
        QTO & \bf{37.6} & \bf{40.1} & \bf{32.3} & \bf{39.3} & \bf{48.2} & \bf{60.7} & \bf{37.7} & \bf{42.4} & \bf{16.8} & \bf{20.5} \\
        \bottomrule
    \end{tabular}
    \caption{MAPE (\%, $\downarrow$) on answer set cardinality prediction, evaluated on FB15k and NELL995.}
    \label{tb:mape1}
\end{table}

Table~\ref{tb:mape1} reports the mean absolute percentage error (MAPE) on cardinality prediction on FB15k and NELL995.
We see that QTO outperforms GNN-QE by a large margin on these two datasets, and the advantage is consistent across all types of queries.

%% file: main.bbl
\begin{thebibliography}{36}
\providecommand{\natexlab}[1]{#1}
\providecommand{\url}[1]{\texttt{#1}}
\expandafter\ifx\csname urlstyle\endcsname\relax
  \providecommand{\doi}[1]{doi: #1}\else
  \providecommand{\doi}{doi: \begingroup \urlstyle{rm}\Url}\fi

\bibitem[Arakelyan et~al.(2021)Arakelyan, Daza, Minervini, and
  Cochez]{arakelyan2021complex}
Arakelyan, E., Daza, D., Minervini, P., and Cochez, M.
\newblock \href{https://openreview.net/pdf?id=Mos9F9kDwkz}{Complex query
  answering with neural link predictors}.
\newblock In \emph{International Conference on Learning Representations}, 2021.

\bibitem[Bai et~al.(2021)Bai, Ying, Ren, and Leskovec]{bai2021cone}
Bai, Y., Ying, Z., Ren, H., and Leskovec, J.
\newblock
  \href{https://proceedings.neurips.cc/paper/2021/file/662a2e96162905620397b19c9d249781-Paper.pdf}{Modeling
  heterogeneous hierarchies with relation-specific hyperbolic cones}.
\newblock \emph{Advances in Neural Information Processing Systems},
  34:\penalty0 12316--12327, 2021.

\bibitem[Bai et~al.(2022)Bai, Lv, Li, Hou, Qu, Dai, and Xiong]{bai2022squire}
Bai, Y., Lv, X., Li, J., Hou, L., Qu, Y., Dai, Z., and Xiong, F.
\newblock \href{https://arxiv.org/pdf/2201.06206.pdf}{SQUIRE: A
  sequence-to-sequence framework for multi-hop knowledge graph reasoning}.
\newblock In \emph{EMNLP}, 2022.

\bibitem[Bala{\v{z}}evi{\'c} et~al.(2019)Bala{\v{z}}evi{\'c}, Allen, and
  Hospedales]{balavzevic2019tucker}
Bala{\v{z}}evi{\'c}, I., Allen, C., and Hospedales, T.
\newblock \href{https://aclanthology.org/D19-1522.pdf}{TuckER: Tensor
  factorization for knowledge graph completion}.
\newblock In \emph{Proceedings of the 2019 Conference on Empirical Methods in
  Natural Language Processing and the 9th International Joint Conference on
  Natural Language Processing (EMNLP-IJCNLP)}, pp.\  5185--5194, 2019.

\bibitem[Bordes et~al.(2013)Bordes, Usunier, Garcia-Duran, Weston, and
  Yakhnenko]{bordes2013translating}
Bordes, A., Usunier, N., Garcia-Duran, A., Weston, J., and Yakhnenko, O.
\newblock
  \href{https://proceedings.neurips.cc/paper/2013/file/1cecc7a77928ca8133fa24680a88d2f9-Paper.pdf}{Translating
  embeddings for modeling multi-relational data}.
\newblock \emph{Advances in neural information processing systems}, 26, 2013.

\bibitem[Chami et~al.(2020)Chami, Wolf, Juan, Sala, Ravi, and
  R{\'e}]{chami2020low}
Chami, I., Wolf, A., Juan, D.-C., Sala, F., Ravi, S., and R{\'e}, C.
\newblock \href{https://aclanthology.org/2020.acl-main.617.pdf}{Low-dimensional
  hyperbolic knowledge graph embeddings}.
\newblock In \emph{ACL}, 2020.

\bibitem[Chen et~al.(2022)Chen, Hu, and Sun]{chen2022fuzzy}
Chen, X., Hu, Z., and Sun, Y.
\newblock
  \href{https://ojs.aaai.org/index.php/AAAI/article/view/20310/20069}{Fuzzy
  logic based logical query answering on knowledge graphs}.
\newblock In \emph{Proceedings of the AAAI Conference on Artificial
  Intelligence}, volume~36, pp.\  3939--3948, 2022.

\bibitem[Chen et~al.(2021)Chen, Minervini, Riedel, and
  Stenetorp]{chen2021relation}
Chen, Y., Minervini, P., Riedel, S., and Stenetorp, P.
\newblock \href{https://openreview.net/pdf?id=Qa3uS3H7-Le}{Relation prediction
  as an auxiliary training objective for improving multi-relational graph
  representations}.
\newblock In \emph{3rd Conference on Automated Knowledge Base Construction},
  2021.

\bibitem[Choudhary et~al.(2021)Choudhary, Rao, Katariya, Subbian, and
  Reddy]{choudhary2021self}
Choudhary, N., Rao, N., Katariya, S., Subbian, K., and Reddy, C.~K.
\newblock \href{https://arxiv.org/pdf/2012.13023.pdf}{Self-supervised
  hyperboloid representations from logical queries over knowledge graphs}.
\newblock In \emph{Proceedings of the Web Conference 2021}, pp.\  1373--1384,
  2021.

\bibitem[Dalvi \& Suciu(2007)Dalvi and Suciu]{dalvi2007efficient}
Dalvi, N. and Suciu, D.
\newblock \href{https://www.vldb.org/conf/2004/RS22P1.PDF}{Efficient query
  evaluation on probabilistic databases}.
\newblock \emph{The VLDB Journal}, 16\penalty0 (4):\penalty0 523--544, 2007.

\bibitem[Das et~al.(2018)Das, Dhuliawala, Zaheer, Vilnis, Durugkar,
  Krishnamurthy, Smola, and McCallum]{das2018go}
Das, R., Dhuliawala, S., Zaheer, M., Vilnis, L., Durugkar, I., Krishnamurthy,
  A., Smola, A., and McCallum, A.
\newblock \href{https://openreview.net/pdf?id=Syg-YfWCW}{Go for a walk and
  arrive at the answer: reasoning over paths in knowledge bases using
  reinforcement learning}.
\newblock In \emph{International Conference on Learning Representations}, 2018.

\bibitem[Dettmers et~al.(2018)Dettmers, Minervini, Stenetorp, and
  Riedel]{dettmers2018convolutional}
Dettmers, T., Minervini, P., Stenetorp, P., and Riedel, S.
\newblock
  \href{https://ojs.aaai.org/index.php/AAAI/article/view/11573/11432}{Convolutional
  2d knowledge graph embeddings}.
\newblock In \emph{Proceedings of the AAAI conference on artificial
  intelligence}, volume~32, 2018.

\bibitem[Guu et~al.(2015)Guu, Miller, and Liang]{guu2015traversing}
Guu, K., Miller, J., and Liang, P.
\newblock \href{https://aclanthology.org/D15-1038.pdf}{Traversing knowledge
  graphs in vector space}.
\newblock In \emph{Proceedings of the 2015 Conference on Empirical Methods in
  Natural Language Processing}, pp.\  318--327, 2015.

\bibitem[H{\'a}jek(2013)]{hajek2013metamathematics}
H{\'a}jek, P.
\newblock \emph{Metamathematics of fuzzy logic}, volume~4.
\newblock Springer Science \& Business Media, 2013.

\bibitem[Hamilton et~al.(2018)Hamilton, Bajaj, Zitnik, Jurafsky, and
  Leskovec]{hamilton2018embedding}
Hamilton, W., Bajaj, P., Zitnik, M., Jurafsky, D., and Leskovec, J.
\newblock
  \href{https://proceedings.neurips.cc/paper/2018/file/ef50c335cca9f340bde656363ebd02fd-Paper.pdf}{Embedding
  logical queries on knowledge graphs}.
\newblock \emph{Advances in neural information processing systems}, 31, 2018.

\bibitem[Klement et~al.(2013)Klement, Mesiar, and Pap]{klement2013triangular}
Klement, E.~P., Mesiar, R., and Pap, E.
\newblock \emph{Triangular norms}, volume~8.
\newblock Springer Science \& Business Media, 2013.

\bibitem[Klir \& Yuan(1995)Klir and Yuan]{klir1995fuzzy}
Klir, G. and Yuan, B.
\newblock \emph{Fuzzy sets and fuzzy logic}, volume~4.
\newblock Prentice hall New Jersey, 1995.

\bibitem[Lacroix et~al.(2018)Lacroix, Usunier, and
  Obozinski]{lacroix2018canonical}
Lacroix, T., Usunier, N., and Obozinski, G.
\newblock
  \href{http://proceedings.mlr.press/v80/lacroix18a/lacroix18a.pdf}{Canonical
  tensor decomposition for knowledge base completion}.
\newblock In \emph{International Conference on Machine Learning}, pp.\
  2863--2872. PMLR, 2018.

\bibitem[Lin et~al.(2018)Lin, Socher, and Xiong]{lin2018multi}
Lin, X.~V., Socher, R., and Xiong, C.
\newblock \href{https://aclanthology.org/D18-1362.pdf}{Multi-hop knowledge
  graph reasoning with reward shaping}.
\newblock In \emph{Proceedings of the 2018 Conference on Empirical Methods in
  Natural Language Processing}, pp.\  3243--3253, 2018.

\bibitem[Lv et~al.(2019)Lv, Gu, Han, Hou, Li, and Liu]{lv2019adapting}
Lv, X., Gu, Y., Han, X., Hou, L., Li, J., and Liu, Z.
\newblock \href{https://aclanthology.org/D19-1334.pdf}{Adapting meta knowledge
  graph information for multi-hop reasoning over few-shot relations}.
\newblock In \emph{Proceedings of the 2019 Conference on Empirical Methods in
  Natural Language Processing and the 9th International Joint Conference on
  Natural Language Processing (EMNLP-IJCNLP)}, pp.\  3376--3381, 2019.

\bibitem[Meilicke et~al.(2019)Meilicke, Chekol, Ruffinelli, and
  Stuckenschmidt]{chris2019anyburl}
Meilicke, C., Chekol, M.~W., Ruffinelli, D., and Stuckenschmidt, H.
\newblock \href{https://dl.acm.org/doi/abs/10.5555/3367471.3367477}{Anytime
  bottom-up rule learning for knowledge graph completion}.
\newblock In \emph{Proceedings of the 28th International Joint Conference on
  Artificial Intelligence}, pp.\  3137--3143, 2019.

\bibitem[Pearl(1982)]{pearl1982reverend}
Pearl, J.
\newblock \href{https://www.aaai.org/Papers/AAAI/1982/AAAI82-032.pdf}{Reverend
  bayes on inference engines: A distributed hierarchical approach}.
\newblock In \emph{AAAI}, 1982.

\bibitem[Pearl(2009)]{pearl2009causality}
Pearl, J.
\newblock \emph{Causality}.
\newblock Cambridge university press, 2009.

\bibitem[Ren \& Leskovec(2020)Ren and Leskovec]{ren2020beta}
Ren, H. and Leskovec, J.
\newblock
  \href{https://proceedings.neurips.cc/paper/2020/file/e43739bba7cdb577e9e3e4e42447f5a5-Paper.pdf}{Beta
  embeddings for multi-hop logical reasoning in knowledge graphs}.
\newblock \emph{Advances in Neural Information Processing Systems},
  33:\penalty0 19716--19726, 2020.

\bibitem[Ren et~al.(2019)Ren, Hu, and Leskovec]{ren2019query2box}
Ren, H., Hu, W., and Leskovec, J.
\newblock \href{https://openreview.net/pdf?id=BJgr4kSFDS}{Query2box: reasoning
  over knowledge graphs in vector space using box embeddings}.
\newblock In \emph{International Conference on Learning Representations}, 2019.

\bibitem[Safavi et~al.(2020)Safavi, Koutra, and Meij]{safavi2020evaluating}
Safavi, T., Koutra, D., and Meij, E.
\newblock \href{https://aclanthology.org/2020.emnlp-main.667.pdf}{Evaluating
  the calibration of knowledge graph embeddings for trustworthy link
  prediction}.
\newblock In \emph{Proceedings of the 2020 Conference on Empirical Methods in
  Natural Language Processing (EMNLP)}, pp.\  8308--8321, 2020.

\bibitem[Schlichtkrull et~al.(2018)Schlichtkrull, Kipf, Bloem, Berg, Titov, and
  Welling]{schlichtkrull2018modeling}
Schlichtkrull, M., Kipf, T.~N., Bloem, P., Berg, R. v.~d., Titov, I., and
  Welling, M.
\newblock \href{https://arxiv.org/pdf/1703.06103.pdf}{Modeling relational data
  with graph convolutional networks}.
\newblock In \emph{European semantic web conference}, pp.\  593--607. Springer,
  2018.

\bibitem[Sun et~al.(2018)Sun, Deng, Nie, and Tang]{sun2019rotate}
Sun, Z., Deng, Z.-H., Nie, J.-Y., and Tang, J.
\newblock \href{https://openreview.net/pdf?id=HkgEQnRqYQ}{RotatE: Knowledge
  graph embedding by relational rotation in complex space}.
\newblock In \emph{International Conference on Learning Representations}, 2018.

\bibitem[Toutanova \& Chen(2015)Toutanova and Chen]{toutanova2015observed}
Toutanova, K. and Chen, D.
\newblock \href{https://aclanthology.org/W15-4007.pdf}{Observed versus latent
  features for knowledge base and text inference}.
\newblock In \emph{Proceedings of the 3rd workshop on continuous vector space
  models and their compositionality}, pp.\  57--66, 2015.

\bibitem[Trouillon et~al.(2016)Trouillon, Welbl, Riedel, Gaussier, and
  Bouchard]{trouillon2016complex}
Trouillon, T., Welbl, J., Riedel, S., Gaussier, {\'E}., and Bouchard, G.
\newblock \href{http://proceedings.mlr.press/v48/trouillon16.pdf}{Complex
  embeddings for simple link prediction}.
\newblock In \emph{ICML}, 2016.

\bibitem[Xiong et~al.(2017)Xiong, Hoang, and Wang]{xiong2017deeppath}
Xiong, W., Hoang, T., and Wang, W.~Y.
\newblock \href{https://aclanthology.org/D17-1060.pdf}{DeepPath: A
  reinforcement learning method for knowledge graph reasoning}.
\newblock In \emph{Proceedings of the 2017 Conference on Empirical Methods in
  Natural Language Processing}, pp.\  564--573, 2017.

\bibitem[Yang et~al.(2017)Yang, Yang, and Cohen]{yang2017differentiable}
Yang, F., Yang, Z., and Cohen, W.~W.
\newblock
  \href{https://proceedings.neurips.cc/paper/2017/file/0e55666a4ad822e0e34299df3591d979-Paper.pdf}{Differentiable
  learning of logical rules for knowledge base reasoning}.
\newblock \emph{Advances in neural information processing systems}, 30, 2017.

\bibitem[Zhang et~al.(2021)Zhang, Wang, Chen, Ji, and Wu]{zhang2021cone}
Zhang, Z., Wang, J., Chen, J., Ji, S., and Wu, F.
\newblock
  \href{https://proceedings.neurips.cc/paper/2021/file/a0160709701140704575d499c997b6ca-Paper.pdf}{Cone:
  Cone embeddings for multi-hop reasoning over knowledge graphs}.
\newblock \emph{Advances in Neural Information Processing Systems},
  34:\penalty0 19172--19183, 2021.

\bibitem[Zhu et~al.(2021)Zhu, Zhang, Xhonneux, and Tang]{zhu2021neural}
Zhu, Z., Zhang, Z., Xhonneux, L.-P., and Tang, J.
\newblock
  \href{https://proceedings.neurips.cc/paper/2021/file/f6a673f09493afcd8b129a0bcf1cd5bc-Paper.pdf}{Neural
  bellman-ford networks: A general graph neural network framework for link
  prediction}.
\newblock \emph{Advances in Neural Information Processing Systems},
  34:\penalty0 29476--29490, 2021.

\bibitem[Zhu et~al.(2022)Zhu, Galkin, Zhang, and Tang]{zhu2022neural}
Zhu, Z., Galkin, M., Zhang, Z., and Tang, J.
\newblock \href{https://arxiv.org/pdf/2205.10128.pdf}{Neural-symbolic models
  for logical queries on knowledge graphs}.
\newblock In \emph{Proceedings of the 39th International Conference on Machine
  Learning}, volume 162, pp.\  27454--27478, 2022.

\bibitem[Zou et~al.(2011)Zou, Mo, Chen, {\"O}zsu, and Zhao]{zou2011gstore}
Zou, L., Mo, J., Chen, L., {\"O}zsu, M.~T., and Zhao, D.
\newblock \href{http://vldb.org/pvldb/vol4/p482-zou.pdf}{gStore: answering
  SPARQL queries via subgraph matching}.
\newblock \emph{Proceedings of the VLDB Endowment}, 4\penalty0 (8):\penalty0
  482--493, 2011.

\end{thebibliography}
